\documentclass{article}
\usepackage[preprint]{staix_2026}
\usepackage[utf8]{inputenc}
\usepackage[T1]{fontenc}
\usepackage{hyperref}
\usepackage{url}
\usepackage{booktabs}
\usepackage{amsfonts}
\usepackage{nicefrac}
\usepackage{microtype}
\usepackage{xcolor}
\usepackage{amsmath}
\usepackage{amsfonts}
\usepackage{amsmath, amsthm, amssymb}
\usepackage{enumitem}
\usepackage{multirow}
\usepackage{subcaption}
\newtheorem{lemma}{Lemma}
\newtheorem{assumption}{Assumption}
\newtheorem*{assumption*}{Assumption}
\newtheorem{proposition}{Proposition}
\newtheorem{corollary}{Corollary}
\newtheorem{theorem}{Theorem}
\newtheorem{remark}{Remark}
\usepackage{graphicx}
\usepackage{subcaption}
\usepackage{hyperref}
\usepackage{threeparttable}
\usepackage{adjustbox}
\usepackage[most]{tcolorbox}
\usepackage{tabularx}
\usepackage[htt]{hyphenat}

\title{Variance-aware Reward Modeling with Anchor Guidance}

\author{
  Shuxing Fang\textsuperscript{1}
  \quad Ruijian Han\textsuperscript{1}
  \quad Liangyu Zhang\textsuperscript{2}\thanks{
    Corresponding author.
    Authors are listed in alphabetical order.
    Emails:
    \texttt{shuxing.fang@connect.polyu.hk},
    \texttt{ruijian.han@polyu.edu.hk},
    \texttt{zhangliangyu@sufe.edu.cn},
    \texttt{zhoufan@mail.shufe.edu.cn}.
  }
  \quad Fan Zhou\textsuperscript{2}
  \\[0.5em]
  {\normalfont\small
  \textsuperscript{1}Department of Data Science and Artificial Intelligence,
  Hong Kong Polytechnic University
  }
  \\[-0.1em]
  {\normalfont\small
  \textsuperscript{2}School of Statistics and Data Science,
  Shanghai University of Finance and Economics
  }
}

\begin{document}

\maketitle

\begin{abstract}
Standard Bradley--Terry (BT) reward models are limited when human preferences are pluralistic. Although soft preference labels preserve disagreement information, BT can only  express it by shrinking reward margins.
Gaussian reward models provide an alternative by jointly predicting a reward mean and a reward variance, but suffer from a fundamental
non-identifiability from pairwise preferences alone. We propose Anchor Guided Variance-aware Reward Modeling, a framework that resolves this non-identifiability by augmenting preference data with two coarse response-level anchor labels. Building on this, we prove that two anchors are sufficient for identification, develop a joint training objective and establish a non-asymptotic convergence rate for both the estimated reward mean and variance functions. Across simulation studies and four real-world diverging-preference datasets, our method consistently improves reward modeling performance and downstream RLHF, including PPO training and best-of-$N$ selection. Code is available at \url{https://github.com/fsmiu/Anchor-guided-Variance-aware-Reward-Modeling}.
\end{abstract}

\section{Introduction}
Reinforcement Learning from Human Feedback (RLHF) are widely employed to align LLMs with human preferences and avoid generating harmful content \citep{bai2022training,ouyang2022training,lambert2024tulu,hu2024openrlhf,ji2026overview,liu2026reinforcement}. A typical RLHF pipeline first trains a reward model on human preference data, and then uses this reward model to guide policy optimization via algorithms such as PPO \citep{schulman2017proximal} or GRPO \citep{shao2024deepseekmath}. The quality of the reward model is therefore central to the success of the entire pipeline.

As LLMs are used by increasingly diverse populations, alignment should account for pluralistic
human preferences \citep{gordon2022jury,sorensen2024value,sorensen2024position}. A line of recent work shows that annotators with different backgrounds
frequently provide diverging preferences on the same prompt–response pair \citep{baumler2023examples,mishra2023ai,siththaranjandistributional,zhang2025diverging,halpern2026pairwise}. Recent multi-annotator preference datasets such as \texttt{MultiPref} \citep{miranda2025hybrid}, \texttt{HelpSteer2} \citep{wang2025helpsteerpreference}, \texttt{HelpSteer3} \citep{wang2025helpsteer3preferenceopenhumanannotatedpreference}, and \texttt{PersonalLLM} \citep{zollo2025personalllm} further confirm that a substantial fraction of pairwise comparisons receive split votes from human annotators or LLM judges.
Standard reward modeling approaches simply aggregate these split votes into a single majority label \citep{casper2023open, wang2024helpsteer2}. This aggregation design fails to capture the underlying diversity of preferences and can further lead to preference collapse \citep{zhang2025diverging,halpern2026pairwise} and reward hacking \citep{gao2023scaling,fu2025reward,sun2025probabilistic}.

A natural alternative is to adopt soft preference labels that reflect the empirical proportion of annotators favoring one response, thereby preserving annotator disagreement information and providing a better training signal for learning calibrated preference probabilities \citep{liu2025reward, halpern2026pairwise}. Although soft preference labels can be directly used in the BT \citep{bradley1952rank} loss, the standard BT captures only the mean preference probability.
As a result, annotator disagreement is absorbed by shrinking the reward margin between two responses, rather than being explicitly represented as uncertainty.

A principled alternative is to model rewards as distributions rather than scalars \citep{siththaranjandistributional,lou2024uncertainty,sun2025probabilistic,zhang2025diverging,liu2025uncertaintyquantificationlargelanguage,frick2025reward,chen2026learning}. A representative instance is the Gaussian reward model, which outputs a reward mean and a reward variance \citep{siththaranjandistributional,zhang2025diverging,frick2025reward}. It corresponds to the classical Thurstonian formulation \citep{thurstone1927method} and generalizes BT from a scalar reward with fixed noise to a heteroscedastic reward distribution.  However, this additional modeling flexibility introduces a fundamental identifiability problem \citep{liu2025uncertaintyquantificationlargelanguage,frick2025reward}. Different combinations of reward mean and variance functions can induce the same preference probabilities, making the learned uncertainty difficult to compare and interpret.

To address this issue, we propose Anchor-guided Variance-aware
Reward Modeling, which resolves the non-identifiability of Gaussian reward models via an auxiliary anchor dataset: each response receives two coarse binary anchor labels indicating whether its latent utility exceeds two distinct thresholds. We develop a joint training objective that combines a soft-label preference loss and a multinomial cross-entropy loss induced by the two anchors. To verify that performance is robust to the label source, we construct coarse anchor labels from two distinct sources: external reward models and LLM judges. Based on this framework, we prove that the two coarse anchors provide the constraints to uniquely identify both the reward mean and variance functions, and we further establish a non-asymptotic convergence rate for both functions. Empirically, on simulation data and four real-world diverging-preference
benchmarks (\texttt{MultiPref}, \texttt{HelpSteer2},
\texttt{HelpSteer3}, \texttt{PersonalLLM}), our method consistently
outperforms other baselines in both
reward modeling evaluation and downstream RLHF performance (PPO
training and best-of-$N$ selection) .

To summarize, our contributions include: (1) a formal characterization
of the non-identifiability of Gaussian reward models and a two-anchor
identification result; (2) a practical anchor-guided variance-aware reward modeling
framework with finite-sample convergence guarantees for both the
reward mean and variance function; (3) a comprehensive empirical demonstration
of its effectiveness on diverging-preference benchmarks and downstream
RLHF evaluation.

\textbf{Related work.} We focus on three lines of literature most related to our work. \textbf{\emph{(i) Pluralistic alignment.}}

Existing methods for pluralistic preferences can be broadly divided into two categories. The first leverages auxiliary information about annotators, such as demographic attributes, group memberships, or user-specific information, to learn group-specific or personalized reward functions \citep{chakraborty2024maxmin,poddar2024personalizing,chen2024pal,chen2025pal,kim2026swapguided}.
The second represents rewards as distributions rather than scalars, often referred to as distributional or uncertainty-aware reward model. Since this line of work is closely related to ours, we discuss it in detail below. \textbf{\emph{(ii) Distributional reward model.}} Existing distributional framework can be divided into three categories: mean-variance reward models \citep{lou2024uncertainty,siththaranjandistributional,sun2025probabilistic,zhang2025diverging,liu2025uncertaintyquantificationlargelanguage,frick2025reward}, categorical reward models \citep{siththaranjandistributional,zhang2025diverging,chen2026learning}, and ensemble-based reward models \citep{coste2023reward,lou2024uncertainty,halpern2026pairwise}. Among them, mean-variance reward models are the most related to our work. \cite{frick2025reward} only briefly discusses
the unidentifiability of Gaussian reward models and their empirical results show comparable performance to the BT model. In contrast, our work resolves this non-identifiability via two coarse anchor labels and obtains clear empirical improvements over other baselines. \textbf{\emph{(iii) Auxiliary supervision.}} A growing line of work uses auxiliary fine-grained supervision for BT reward modeling, such as teacher reward values \citep{zhang2026distill} or attribute-level scores \citep{zhang2026bradleyterry}. In contrast, we employ external reward models or LLM as a judge solely as a low-cost tool to overcome the lack of reponse-level coarse quality labels (e.g., good / normal / bad) in existing preference data. The closest prior work is \citet{chen2026learning}, which uses such coarse quality labels to train a categorical reward model. From a statistical perspective, we instead use these coarse labels as identification constraints for a Gaussian reward model, and prove both identifiability and a non-asymptotic convergence rate. We further empirically evaluate the robustness of our method across different coarse anchor label sources, and analyze anchor data efficiency by varying proportions of available anchor data.

\section{Preliminaries}
Let $\mathcal{X}$ be the prompt space and $\mathcal{Y}$ be the response space. We consider a pairwise preference dataset $\mathcal{D}_{\mathrm{pref}}
= \{(x_i, y_{i,1}, y_{i,2}, C_i)\}_{i=1}^{n_{\mathrm{pref}}}$ consisting of $n_{\mathrm{pref}}$ i.i.d.\ tuples, where $x \in \mathcal{X}$ is a prompt, $y_{i,1}, y_{i,2} \in \mathcal{Y}$ are two candidate responses and
$C_i \in \{0,1\}$ is the observed binary preference label, with $C_i = 1$
indicating that $y_{i,1}$ is preferred over $y_{i,2}$.

\textbf{BT Reward Model.} Most existing reward-based RLHF algorithms adopt the BT assumption \citep{bradley1952rank}. We assume the latent utility of a response $y$ given a prompt $x$ is $U(x,y) = r_\phi(x,y) + \epsilon$, where $r_\phi: \mathcal{X} \times \mathcal{Y} \to \mathbb{R}$ is a parametrized scalar reward function and $\epsilon$ is i.i.d.\ standard Gumbel noise. The probability that $y_{i,1}$ is preferred to $y_{i,2}$ is
{
\setlength{\abovedisplayskip}{4pt}
\setlength{\belowdisplayskip}{4pt}
\begin{equation}
\mathbb{P}(C_i = 1 \mid x_i, y_{i,1}, y_{i,2})
=
\mathbb{P}\bigl(U(x_i, y_{i,1}) > U(x_i, y_{i,2})\bigr)
=
\sigma\!\bigl(r_\phi(x_i, y_{i,1}) - r_\phi(x_i, y_{i,2})\bigr)
\end{equation}
}
where $\sigma(t)=1/(1+e^{-t})$ is the sigmoid function. In practice, we optimize the reward function by minimizing the negative log-likelihood loss over the preference dataset: $\mathcal{L}_{\mathrm{BT}}(\mathcal{D}_{\mathrm{pref}} \mid \phi) = - \frac{1}{n_{\mathrm{pref}}} \sum_{i=1}^{n_{\mathrm{pref}}} \left[ C_i \log \sigma(z_i) + (1 - C_i) \log \bigl(1 - \sigma(z_i)\bigr) \right]$, where $z_i= r_\phi(x_i, y_{i,1}) - r_\phi(x_i, y_{i,2})$ is the reward margin.

\textbf{Gaussian Reward Model.} Following the Thurstonian model assumption \citep{thurstone1927method}, we consider a Gaussian reward model that explicitly captures heteroscedastic uncertainty in human preferences. Specifically, for each $(x,y)$, we assume the latent utility follows $U(x,y) \sim \mathcal{N}\!\bigl(r_\phi(x,y),\, s_\phi^2(x,y)\bigr)$, where $r_\phi(x,y)$ is the reward mean function and $s_\phi(x,y) > 0$ is the reward standard deviation function. For a comparison $(x_i, y_{i,1}, y_{i,2})$, the preference probability is
{
\setlength{\abovedisplayskip}{3pt}
\setlength{\belowdisplayskip}{3pt}
\begin{equation}
\mathbb{P}(C_i = 1 \mid x_i, y_{i,1}, y_{i,2})
=
\mathbb{P}\bigl(U(x_i, y_{i,1}) > U(x_i, y_{i,2})\bigr)
=
\Phi\!\left(
\frac{
r_\phi(x_i, y_{i,1}) - r_\phi(x_i, y_{i,2})
}{
\sqrt{
s_\phi^2(x_i, y_{i,1}) + s_\phi^2(x_i, y_{i,2})
}}
\right),
\label{gaussian_p}
\end{equation}
}
where $\Phi(\cdot)$ denotes the standard normal cumulative distribution
function. Similar to the BT model, we optimize both functions by minimizing the negative log-likelihood over the preference dataset:
{
\setlength{\abovedisplayskip}{3pt}
\setlength{\belowdisplayskip}{3pt}
\begin{equation}
\mathcal{L}_{\mathrm{G}}(\mathcal{D}_{\mathrm{pref}} \mid \phi) = - \frac{1}{n_{\mathrm{pref}}} \sum_{i=1}^{n_{\mathrm{pref}}} \left[ C_i \log \Phi(z_i^{\mathrm{G}}) + (1 - C_i) \log(1 - \Phi(z_i^{\mathrm{G}})) \right],
\label{eq:gaussian-margin}
\end{equation}
}
where
$z_i^{\mathrm{G}} = \bigl(r_\phi(x_i, y_{i,1}) - r_\phi(x_i, y_{i,2})\bigr) / \sqrt{s_\phi^2(x_i, y_{i,1}) + s_\phi^2(x_i, y_{i,2})}$
is the variance-normalized reward margin. The identifiability issue of this formulation is analyzed in
Section~\ref{sec:Identifiability}.

\section{Identifiability}\label{sec:Identifiability}

We first demonstrate that, given only pairwise preference dataset $\mathcal{D}_{\mathrm{pref}}$, the reward mean
and standard deviation functions $(r_\phi, s_\phi)$ cannot be uniquely identified.

\begin{proposition}[Non-identifiability of the Gaussian reward model]
\label{prop:nonid}
The Gaussian preference probability defined in \eqref{gaussian_p}
is invariant under the following two families of function transformations: \emph{(i) Reward translation.} For any constant $b \in \mathbb{R}$, the preference probability is invariant under the mapping $(r_\phi, s_\phi) \mapsto (r_\phi + b,\, s_\phi)$. \emph{(ii) Joint positive scaling.} For any scalar $c > 0$,
the preference probability is invariant under the mapping $(r_\phi, s_\phi) \mapsto (c\, r_\phi,\, c\, s_\phi)$.
\end{proposition}
Thus, different choices of \((r_\phi, s_\phi)\) induce the same distribution over observed preferences. In the BT formulation, the noise distribution is fixed, so $r_\phi$ is identifiable up to an additive constant.
The additional modeling flexibility provided by the
learnable variance $s_\phi^2(x,y)$ in the Gaussian model introduces a new identifiability issue. Therefore, without additional constraints or regularization, neither the absolute level of $r_\phi$ nor the joint scale of
$(r_\phi, s_\phi)$ is identifiable only from $\mathcal{D}_{\mathrm{pref}}$.

To resolve the non-identifiability under preference-only data, we
impose additional response-level anchoring constraints. Specifically,
suppose we have access to an additional anchor dataset
$\mathcal{D}_{\mathrm{anc}} = \{(x_j, y_j, A_{j,1}, A_{j,2})\}_{j=1}^{n_{\mathrm{anc}}}$
consisting of $n_{\mathrm{anc}}$ i.i.d.\ tuples, where each prompt--response
pair $(x_j, y_j)$ is associated with two binary labels indicating whether
the response quality exceeds two fixed anchor thresholds $\tau_1 < \tau_2$. Formally, for $k \in \{1, 2\}$, the anchor label is
defined as $A_{j,k} = \mathbf{1}\{U(x_j, y_j) \geq \tau_k\}$. Under the
Gaussian reward model, the probability that the latent utility exceeds
$\tau_k$ is
{
\setlength{\abovedisplayskip}{3pt}
\setlength{\belowdisplayskip}{3pt}
\begin{equation}
\label{eq:anchor-prob}
    q_{k}(x,y)
    = \mathbb{P}\bigl(U(x,y) \geq \tau_k \mid x,y\bigr)
    = \Phi\!\left(\frac{r_\phi(x,y) - \tau_k}{s_\phi(x,y)}\right),
    \qquad k \in \{1, 2\}.
\end{equation}
}

 We now show that two distinct anchor thresholds are sufficient to
identify the function pair $(r_\phi, s_\phi)$.

\begin{proposition}[Identifiability under two anchors]
\label{prop:two-anchor-id}
Let $\tau_1 < \tau_2$ be two distinct anchor thresholds.
Under the Gaussian reward model, the two anchor probability functions
$q_{1}(x,y)$ and $q_{2}(x,y)$ uniquely identify the reward mean function
$r_\phi(x,y)$ and the standard deviation function $s_\phi(x,y)$ on the support of the prompt--response distribution.
\end{proposition}

The intuition behind Proposition~\ref{prop:two-anchor-id} is simple: for each prompt--response pair $(x,y)$, the two anchor probabilities provide two constraints for the two unknowns
$r_\phi(x,y)$ and $s_\phi(x,y)$. Specifically, we obtain the following equations $r_\phi(x,y) -\tau_1=s_\phi(x,y)\, z_{\mathrm{anc},1}(x,y)$ and $r_\phi(x,y) - \tau_2 =s_\phi(x,y)\, z_{\mathrm{anc},2}(x,y)$, where
$z_{\mathrm{anc},k}(x,y)= \Phi^{-1}\bigl(q_{k}(x,y)\bigr)$ for $k \in \{1, 2\}$. Solving for $s_\phi(x,y)$ and  $r_\phi(x,y)$ yields the following closed-form solutions: $s_\phi(x,y) =\frac{\tau_2 - \tau_1}{z_{\mathrm{anc},1}(x,y) - z_{\mathrm{anc},2}(x,y)}$ and $r_\phi(x,y) =\tau_1 + s_\phi(x,y)\, z_{\mathrm{anc},1}(x,y) =\tau_2 + s_\phi(x,y)\, z_{\mathrm{anc},2}(x,y)$. The solution is unique as long as $z_{\mathrm{anc},1}(x,y) \neq z_{\mathrm{anc},2}(x,y)$. Because $\tau_1 \neq \tau_2$ and $s_\phi(x,y)>0$, the condition $q_{1}(x,y) \neq q_{2}(x,y)$ is naturally satisfied. Since the same argument applies to every prompt--response pair on the support, the two anchor probability functions identify both $r_\phi$ and $s_\phi$. Therefore, two distinct anchors are sufficient. Additionally, we design a simulation study to empirically confirm that, unlike preference-only baselines, the two-anchor model recovers both the reward mean and variance (see Figure~\ref{fig:anchor-recovery} in Appendix~\ref{app:sim}).

\section{Methodology}\label{sec:method}

We introduce our proposed Anchor-guided
Variance-aware Reward Modeling framework.

\textbf{Model architecture.}
We parameterize the Gaussian reward model using a shared backbone
followed by two linear heads. The shared backbone is an instruction-tuned
LLM with its language model head removed. Given a prompt--response pair $(x,y)$, the shared backbone with parameter $\theta_{\mathrm{backbone}}$ produces a hidden representation $\psi(x,y) \in \mathbb{R}^d$. Two separate linear heads then map this
representation to the reward mean and variance: $r_\phi(x,y) = \mathbf{w}_r^\top \psi(x,y)$ and $s_\phi^2(x,y)=\operatorname{softplus}\!\bigl(\mathbf{w}_v^\top \psi(x,y)\bigr)$, where $\phi = \{\theta_{\mathrm{backbone}}, \mathbf{w}_r, \mathbf{w}_v\}$
denotes the full parameter set. We adopt the $\operatorname{softplus}$ activation to ensure $s_\phi^2(x,y) > 0$, a common choice for enforcing positivity in neural heteroscedastic modeling \citep{skafte2019reliable,seitzer2022pitfalls,stirn2023faithful}.

\textbf{Anchor label generation.}
To resolve the identifiability issue discussed in Section~\ref{sec:Identifiability}, we require additional coarse anchor labels at the response level. Such labels can be viewed as coarse quality assessments indicating whether a response is above a low or high quality threshold (e.g., good / normal / bad). However, existing preference datasets mainly focus on pairwise comparison labels and often lack response-level coarse quality labels. We therefore design two strategies to automatically construct coarse anchor labels. (1) \textit{\textbf{External reward model}}: For a prompt--response pair $(x_i,y_i)$, we compute a scalar score $\tilde{r}(x_i,y_i)$ from an external reward model. (2) \textit{\textbf{Human/LLM-as-a-judge.}} Each response can be evaluated by strong language models or human annotators along
multiple quality dimensions. These dimension-wise scores are then aggregated into a scalar reward score $\tilde{r}(x_i, y_i)$ via a weighted sum. For both strategies, we then compute the normalized scalar reward $\tilde{r}'(x_i,y_i)$ by subtracting the training set mean. The two thresholds $\tau_1 < \tau_2$ are then set to the $q$-th and $(1-q)$-th quantiles of the centered score distribution, yielding the anchor labels $A_{i,1} = \mathbf{1}\{\tilde{r}'(x_i, y_i) \geq \tau_1\}$ and $A_{i,2} = \mathbf{1}\{\tilde{r}'(x_i, y_i) \geq \tau_2\}$. Details are provided in Appendix~\ref{app:anchor_details}.

\textbf{Joint training objective.}
Human preferences often diverge across annotators \citep{sorensen2024position,zhang2025diverging,halpern2026pairwise}. Following recent work that treats annotator disagreement as informative soft labels \citep{halpern2026pairwise}, we train the preference model with soft labels rather than collapsing multiple annotations into a single hard label. Specifically, if the $i$-th comparison
$(x_i,y_{i,1},y_{i,2})$ is annotated by $m_i$ annotators, let
$C_{i,a}\in\{0,1\}$ denote the binary preference label provided by the
$a$-th annotator, where $C_{i,a}=1$ indicates that $y_{i,1}$ is preferred
over $y_{i,2}$. We define the soft preference label as $\bar{C}_i = \frac{1}{m_i}\sum_{a=1}^{m_i} C_{i,a}$, where $\bar{C}_i$ represents the empirical fraction of annotators who
prefer $y_{i,1}$ over $y_{i,2}$. Given the aggregated soft preference dataset
$\{(x_i, y_{i,1}, y_{i,2}, \bar{C}_i)\}_{i=1}^{n_{\mathrm{pref}}}$,
the negative log-likelihood is
{
\setlength{\abovedisplayskip}{3pt}
\setlength{\belowdisplayskip}{3pt}
\begin{equation}
\label{eq:pref-loss}
    \mathcal{L}_{\mathrm{pref}}(\phi)
    =
    -\frac{1}{n_{\mathrm{pref}}}
    \sum_{i=1}^{n_{\mathrm{pref}}}
    \Bigl[
    \bar{C}_i \log \Phi(z_i^{\mathrm{G}})
    +
    (1-\bar{C}_i)\log\bigl(1-\Phi(z_i^{\mathrm{G}})\bigr)
    \Bigr],
\end{equation}
}
where $z_i^{\mathrm{G}}$ is the variance-normalized reward margin defined in Eq.~\eqref{eq:gaussian-margin} Given an anchor dataset
$\mathcal{D}_{\mathrm{anc}}
= \{(x_j, y_j, A_{j,1}, A_{j,2})\}_{j=1}^{n_{\mathrm{anc}}}$,
where $A_{j,k}=\mathbf{1}\{U(x_j,y_j)\ge \tau_k\}$ and
$\tau_1<\tau_2$, the two anchor labels are derived from the same latent utility
$U(x_j, y_j)$ and thus satisfy the ordinal constraint
$A_{j,2} \leq A_{j,1}$. Specifically, the valid label configurations are
$(A_{j,1},A_{j,2})=(0,0)$, $(1,0)$, and $(1,1)$, corresponding to
$U(x_j,y_j)<\tau_1$, $\tau_1\le U(x_j,y_j)<\tau_2$, and
$U(x_j,y_j)\ge \tau_2$. The probabilities of three threshold-induced classes are $\pi_{j,0}=\mathbb{P}(A_{j,1}=0, A_{j,2}=0) = 1 - q_{1}(x_j, y_j)$, $\pi_{j,1}=\mathbb{P}(A_{j,1}=1, A_{j,2}=0) = q_{1}(x_j, y_j) - q_{2}(x_j, y_j)$ and $\pi_{j,2}=\mathbb{P}(A_{j,1}=1, A_{j,2}=1) = q_{2}(x_j, y_j)$, where $q_{k}(x,y)$ is the conditional anchor probability defined in Eq.~\eqref{eq:anchor-prob}.
Thus, the anchor negative log-likelihood is the multinomial
cross-entropy over the three threshold-induced classes:
{
\setlength{\abovedisplayskip}{3pt}
\setlength{\belowdisplayskip}{3pt}
\begin{equation}
\label{eq:anchor-loss}
\mathcal{L}_{\mathrm{anc}}(\phi)
    =
    -\frac{1}{n_{\mathrm{anc}}}
    \sum_{j=1}^{n_{\mathrm{anc}}}
    \Bigl[
    (1-A_{j,1}) \log \pi_{j,0}
    +
    (A_{j,1}-A_{j,2}) \log \pi_{j,1}
    +
    A_{j,2} \log \pi_{j,2}
    \Bigr].
\end{equation}
}
The joint training objective is to minimize the combined negative
log-likelihood: $\mathcal{L}(\phi)
    = \mathcal{L}_{\mathrm{pref}}(\phi) + \lambda\, \mathcal{L}_{\mathrm{anc}}(\phi)$, where $\lambda > 0$ is a weighting hyperparameter to control the strength of anchor supervision.

\section{Theory}

For notational simplicity, we define $u=\psi(x,y)\in\mathcal{U}\subset\mathbb{R}^d$,
where $\mathcal{U}$ is a compact domain. We
then specify the latent reward mean and log-standard-deviation functions as
$r(u)$ and $g(u)=\log s(u)$ in our theoretical analysis.
 We adopt the log-standard-deviation parameterization instead of the softplus form employed in
Section~\ref{sec:method}, since the log parameterization ensures that $g$ lies in an unconstrained
function space and is more convenient for sieve analysis.
Let $\{\mathcal{R}_K\}_{K\ge 1}$ and
$\{\mathcal{G}_K\}_{K\ge 1}$ be sequences of $K$-dimensional sieve spaces
(e.g., spline spaces) defined on $\mathcal{U}$. For a fixed constant $B_\Theta>B$, we define the truncated sieve space $\Theta_K(B_\Theta)=\left\{
(r,g)\in\mathcal R_K\times\mathcal G_K:
\|r\|_\infty\le B_\Theta,\
\|g\|_\infty\le B_\Theta
\right\}$.
We consider the aggregated soft-label
preference dataset
$\{(x_i,y_{i,1},y_{i,2},\bar C_i)\}_{i=1}^{n_{\mathrm{pref}}}$
defined in Section~\ref{sec:method}. Under the true functions
$(r^*,g^*)$, we assume $\mathbb E[\bar C_i\mid x_i,y_{i,1},y_{i,2}]
=
\Phi\!\left(z_i^{\mathrm G}(r^*,g^*)\right)$. The anchor label pair $(A_{j,1},A_{j,2})$ follows
the ordinal three-class multinomial distribution induced by the two
thresholds $\tau_1<\tau_2$. We estimate the true functions by minimizing the joint empirical loss
{
\setlength{\abovedisplayskip}{3pt}
\setlength{\belowdisplayskip}{3pt}
\[
(\hat r,\hat g)
\in
\arg\min_{(r,g)\in\Theta_K(B_\Theta)}
\Bigl\{
\mathcal L_{\mathrm{pref}}(r,g)
+
\lambda\,\mathcal L_{\mathrm{anc}}(r,g)
\Bigr\},
\]
}
where $\mathcal L_{\mathrm{pref}}(r,g)$ is the empirical preference loss defined in Eq.~\eqref{eq:pref-loss}, $\mathcal L_{\mathrm{anc}}(r,g)$ is the empirical anchor loss
defined in Eq.~\eqref{eq:anchor-loss}, and $\lambda>0$ is a weighting parameter.

We next introduce some assumptions to analyze this estimator.

\begin{assumption}[Data generation]
\label{ass:data}
Let $\mathbb{P}_\psi$ denote the marginal distribution of $\psi(x,y)$. For each anchor data $(x,y)$, we have $\psi(x,y)\sim\mathbb{P}_\psi$. For each preference data $(x,y_1,y_2)$, the two responses are conditionally i.i.d.\ given $x$, we have $\psi(x,y_1),\psi(x,y_2)\sim\mathbb{P}_\psi$.
\end{assumption}

\begin{assumption}[Structure and regularity]
\label{ass:1}
 The true functions $(r^*,g^*)$ are uniformly bounded by some $B>0$ and lie in H\"older classes with smoothness indices $\alpha_r,\alpha_g>0$.
\end{assumption}

\begin{assumption}[Sieve approximation and stability]
\label{ass:2}
The sieve spaces $\mathcal R_K$ and $\mathcal G_K$ approximate $r^*$ and $g^*$
at the standard nonparametric rates and the basis functions are uniformly stable.

\end{assumption}

\begin{assumption}[Sieve dimension]
\label{ass:dim}
The sieve dimension $K$ is allowed to grow with the sample sizes and satisfies $\frac{K}{\min\{n_{\mathrm{pref}},n_{\mathrm{anc}}\}}\to 0$.
\end{assumption}

More detailed versions of Assumption~\ref{ass:data}--\ref{ass:2} are provided in
Appendix~\ref{app:additional}. Assumption \ref{ass:1} requires the feature representation to be uniformly bounded and non-degenerate, and the true functions to be smooth and uniformly bounded.
 These conditions are standard in both the RLHF literature \citep{zhu2023principled,liu2025uncertaintyquantificationlargelanguage,zhang2026bradleyterry} and nonparametric sieve estimation \citep{shen1994convergence,newey1997convergence,chen2007large}.
 Assumption~\ref{ass:2} ensures that the true functions can be well approximated by finite-dimensional sieve spaces, while Assumption~\ref{ass:dim} controls the growth of the sieve dimension relative to the sample sizes.

 \begin{proposition}[Global quadratic curvature of the two-anchor population loss]
\label{prop:anchor-global-curvature}
We define the population anchor loss $\mathcal M_{\mathrm{anchor}}(r,g)
$ and the optimal population anchor loss $\mathcal M_{\mathrm{anchor}}(r^*,g^*)$. Under Assumptions~\ref{ass:1}--\ref{ass:2}, there exists a constant
$c_{\mathrm{anc}}>0$, depending only on $B_\Theta,\tau_1,\tau_2$, such
that for all $(r,g)\in\Theta_K(B_\Theta)$,
{
\setlength{\abovedisplayskip}{3pt}
\setlength{\belowdisplayskip}{3pt}
\[
\mathcal{M}_{\mathrm{anchor}}(r,g)
-
\mathcal{M}_{\mathrm{anchor}}(r^*,g^*)
\ge
c_{\mathrm{anc}}
\left(
\|r-r^*\|_{L_2(\mathbb{P}_\psi)}^2
+
\|g-g^*\|_{L_2(\mathbb{P}_\psi)}^2
\right).
\]
}
\end{proposition}

Without two-anchor labels, the preference loss alone is invariant to certain transformations discussed in Section \ref{sec:Identifiability}. Two distinct anchor thresholds break the degeneracy, and Proposition~\ref{prop:anchor-global-curvature} provides a quadratic lower bound on the anchor population excess loss, which extends to the joint population excess loss (See Corollary~\ref{cor:joint-curvature} in Appendix). This quadratic curvature converts the empirical fluctuation between the empirical and population excess losses into estimation error bounds, yielding the following non-asymptotic rate.

\begin{theorem}[Non-asymptotic rate for the sieve MLE]
\label{thm:main-rate}
Under Assumptions~\ref{ass:data}, \ref{ass:1} and \ref{ass:2}, there exists a constant $C>0$, depending only on
$B_\Theta,\tau_1,\tau_2$, such that for every $t\ge 1$, with
probability at least $1-2e^{-t}$,
{\small
\setlength{\abovedisplayskip}{3pt}
\setlength{\belowdisplayskip}{3pt}
\begin{equation}
\begin{aligned}
\|\hat r-r^*\|_{L_2(\mathbb{P}_\psi)} + \|\hat g-g^*\|_{L_2(\mathbb{P}_\psi)}
&\le C \Big[ \sqrt{1+\lambda^{-1}} \left( K^{-\alpha_r/d} + K^{-\alpha_g/d} \right) + \lambda^{-1}\sqrt{\tfrac{K+t}{n_{\mathrm{pref}}}} + \sqrt{\tfrac{K+t}{n_{\mathrm{anc}}}} \Big]
\end{aligned}
\end{equation}
}

In particular, with $\alpha:=\min\{\alpha_r,\alpha_g\}$,
{\small
\setlength{\abovedisplayskip}{3pt}
\setlength{\belowdisplayskip}{3pt}
\begin{equation}
\|\hat r-r^*\|_{L_2(\mathbb{P}_\psi)} + \|\hat g-g^*\|_{L_2(\mathbb{P}_\psi)} \le C\!\left[ \sqrt{1+\lambda^{-1}}\,K^{-\alpha/d} + \lambda^{-1}\sqrt{\tfrac{K+t}{n_{\mathrm{pref}}}} + \sqrt{\tfrac{K+t}{n_{\mathrm{anc}}}} \right]
\end{equation}
}

\end{theorem}

The bound in Theorem~\ref{thm:main-rate} consists of two parts: a deterministic approximation error $\sqrt{1+\lambda^{-1}}K^{-\alpha/d}$ arising from the sieve approximation of $(r^*,g^*)$, and a stochastic estimation error $\lambda^{-1}\sqrt{(K+t)/n_{\mathrm{pref}}}+\sqrt{(K+t)/n_{\mathrm{anc}}}$ controlled by the smaller of the two sample sizes. Notably, (i) the approximation error vanishes as $K\to\infty$ at the standard nonparametric rate; (ii) Under Assumption~\ref{ass:dim}, the stochastic term vanishes as
\(n_{\min}\to\infty\), and is governed by $n_{\min}$ rather than $n_{\mathrm{pref}}+n_{\mathrm{anc}}$, reflecting the complementary roles of the two datasets. Balancing the two parts via $K\asymp n_{\min}^{d/(2\alpha+d)}$ yields the optimal rate $\mathcal{O}_p(n_{\min}^{-\alpha/(2\alpha+d)})$, which matches the classical minimax rate for nonparametric
regression \citep{stone1982optimal}. Whereas existing RLHF analyses \citep{zhu2023principled} focus on the reward mean function, the proposed estimator additionally recovers the variance function, without sacrificing the classical nonparametric convergence rate under the two-anchor identification condition.  Auxiliary results and Detailed proofs are provided in Appendices~\ref{app:additional} and~\ref{app:proofs}, respectively.

\begin{table}[t]
\centering
\caption{
Overall performance of different models with the \texttt{Llama-3.2-3B-Instruct} as backbone. Bold numbers represent the best performance. Underlined numbers represent the second best.
}
\label{tab:rm_eval_3b}
\scriptsize
\setlength{\tabcolsep}{1.5pt}
\renewcommand{\arraystretch}{0.85}
\resizebox{\textwidth}{!}{
\begin{tabular}{l ccc ccc | l ccc ccc}
\toprule
\multirow{2}{*}{\textbf{Method}} & \multicolumn{3}{c}{\textbf{RewardBench}} & \multicolumn{3}{c|}{\textbf{PPE-P}} &
\multirow{2}{*}{\textbf{Method}} & \multicolumn{3}{c}{\textbf{RewardBench}} & \multicolumn{3}{c}{\textbf{PPE-P}} \\
\cmidrule(lr){2-4} \cmidrule(lr){5-7} \cmidrule(lr){9-11} \cmidrule(l){12-14}
& Acc $\uparrow$ & Brier $\downarrow$ & CE $\downarrow$ & Acc $\uparrow$ & Brier $\downarrow$ & CE $\downarrow$ &
& Acc $\uparrow$ & Brier $\downarrow$ & CE $\downarrow$ & Acc $\uparrow$ & Brier $\downarrow$ & CE $\downarrow$ \\
\midrule

\multicolumn{7}{c|}{\textbf{Dataset: MultiPref-8K}} & \multicolumn{7}{c}{\textbf{Dataset: HelpSteer3-20K}} \\
\midrule
BT                        & 71.7\% & 0.209 & 0.604 & \underline{59.3\%} & \underline{0.236} & \underline{0.665} &
BT                        & 79.4\% & 0.169 & 0.590 & 60.9\% & 0.267 & 0.862 \\
BT(Hard)                  & 75.9\% & 0.194 & 0.703 & 57.1\% & 0.295 & 0.962 &
BT(Hard)                  & 77.7\% & 0.178 & 0.639 & 60.3\% & 0.285 & 1.015 \\
Gaussian                  & 73.2\% & 0.201 & 0.584 & 57.7\% & 0.242 & 0.680 &
Gaussian                  & 78.7\% & 0.168 & 0.644 & 60.3\% & 0.264 & 0.932 \\
Two-anchor-Skywork        & \textbf{82.1\%} & \textbf{0.163} & \textbf{0.490} & \textbf{60.0\%} & \textbf{0.234} & \textbf{0.664} &
Two-anchor-Skywork        & \underline{82.7\%} & \underline{0.147} & \underline{0.506} & \textbf{63.0\%} & \underline{0.244} & 0.782 \\
Two-anchor-DeepSeek-Pro   & 78.4\% & 0.180 & 0.534 & 56.3\% & 0.244 & 0.682 &
Two-anchor-DeepSeek-Pro   & 79.9\% & 0.167 & 0.551 & \underline{61.4\%} & \textbf{0.244} & \textbf{0.732} \\
Two-anchor-DeepSeek-Flash & \underline{80.4\%} & \underline{0.167} & \underline{0.502} & 58.2\% & 0.239 & 0.672 &
Two-anchor-DeepSeek-Flash & \textbf{83.8\%} & \textbf{0.141} & \textbf{0.457} & \underline{61.4\%} & 0.246 & \underline{0.742} \\

\midrule
\multicolumn{7}{c|}{\textbf{Dataset: HelpSteer2-9K}} & \multicolumn{7}{c}{\textbf{Dataset: PersonalLLM-40K}} \\
\midrule
BT                        & \underline{84.8\%} & \underline{0.140} & \underline{0.444} & 58.8\% & 0.251 & 0.719 &
BT                        & 76.8\% & 0.209 & 0.625 & 59.2\% & 0.252 & 0.729 \\
BT(Hard)                  & 81.1\% & 0.149 & 0.509 & 58.4\% & 0.284 & 0.932 &
BT(Hard)                  & 74.2\% & 0.245 & 1.114 & 59.5\% & 0.305 & 1.251 \\
Gaussian                  & 81.9\% & 0.161 & 0.520 & 58.2\% & 0.251 & 0.729 &
Gaussian                  & 73.5\% & 0.220 & 0.644 & 59.5\% & 0.253 & 0.725 \\
Two-anchor-Skywork        & \textbf{85.5\%} & \textbf{0.132} & \textbf{0.413} & \textbf{60.2\%} & 0.241 & 0.694 &
Two-anchor-Skywork        & \textbf{80.3\%} & \textbf{0.181} & \textbf{0.546} & \textbf{61.2\%} & \underline{0.240} & 0.696 \\
Two-anchor-DeepSeek-Pro   & 82.1\% & 0.157 & 0.483 & 59.4\% & \textbf{0.234} & \textbf{0.660} &
Two-anchor-DeepSeek-Pro   & 77.1\% & 0.192 & 0.566 & 60.7\% & 0.241 & \underline{0.688} \\
Two-anchor-DeepSeek-Flash & 83.0\% & 0.151 & 0.467 & \underline{60.0\%} & \underline{0.237} & \underline{0.672} &
Two-anchor-DeepSeek-Flash & \underline{78.4\%} & \underline{0.188} & \underline{0.558} & \underline{60.8\%} & \textbf{0.238} & \textbf{0.680} \\

\bottomrule
\end{tabular}
}
\vspace{-10pt}
\end{table}

\section{Experiments}
\textbf{Datasets.} We evaluate our method on public preference datasets with annotator disagreement. By leveraging multiple annotations per comparison, we construct empirical soft labels that provide a natural signal for modeling preference uncertainty.
After preprocessing public datasets that satisfy this requirement, we obtain four diverging preference datasets with empirical soft labels: \texttt{MultiPref-8K} \citep{miranda2025hybrid}, \texttt{HelpSteer2-9K} \citep{wang2025helpsteerpreference}, \texttt{HelpSteer3-20K} \citep{wang2025helpsteer3preferenceopenhumanannotatedpreference}, \texttt{PersonalLLM-40K} \citep{zollo2025personalllm}. These datasets vary in both size and soft-label distributions, covering different levels of preference disagreement (see Figure~\ref{fig:soft-label-dist} in Appendix). Each dataset is split into training and held-out validation sets. Further details about dataset preprocessing and training are provided in Appendix~\ref{app:rlhf}.

\textbf{Baselines.} We employ the Llama-3 series of models (\texttt{Llama-3.1-8B-Instruct}, \texttt{Llama-3.2-3B-Instruct} and \texttt{Llama-3.2-1B-Instruct}) as the backbone for training reward model. We compare four reward modeling approaches. (1) \textbf{BT}: a standard reward model with the BT loss. (2) \textbf{BT(Hard)}: a BT reward model trained on hard label derived from majority voting (3) \textbf{Gaussian}: a standard gaussian reward model predicting both mean $r_\phi$ and standard deviation $s_\phi$. (4) \textbf{Two-anchor variants}: our proposed model, extending Gaussian with two anchors. We implement three variants: \textbf{Two-anchor-Skywork}, utilizing \texttt{Skywork-Reward-V2-Llama-3.2-3B} as the external reward model; \textbf{Two-anchor-DeepSeek-Pro} and \textbf{Two-anchor-DeepSeek-Flash}, employing \texttt{DeepSeek-V4-Pro} and \texttt{DeepSeek-V4-Flash} as LLM judges.

\begin{figure}[t]
    \centering

    \begin{subfigure}[b]{0.22\textwidth}
        \centering
        \includegraphics[width=\linewidth]{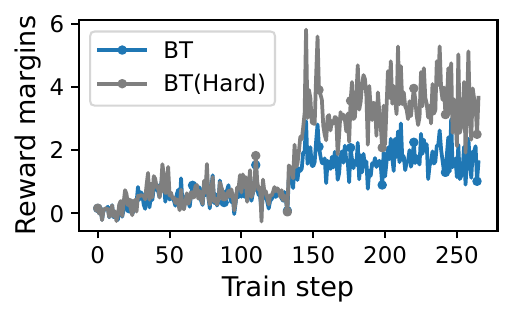}
        \caption{}
        \label{fig:margin}
    \end{subfigure}
    \hfill
    \begin{subfigure}[b]{0.22\textwidth}
        \centering
        \includegraphics[width=\linewidth]{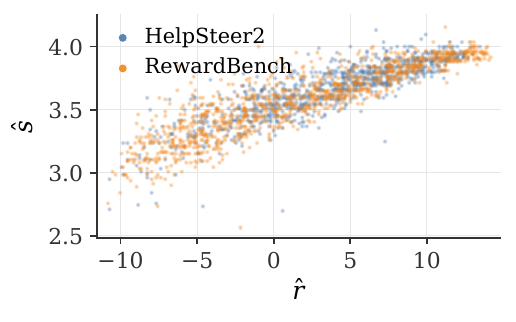}
        \caption{}
        \label{fig:mv_corr}
    \end{subfigure}
    \hfill
    \begin{subfigure}[b]{0.22\textwidth}
        \centering
        \includegraphics[width=\linewidth]{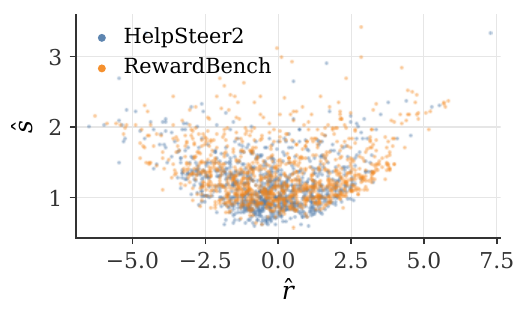}
        \caption{}
        \label{fig:mvanchor_corr}
    \end{subfigure}
    \hfill
\begin{subfigure}[b]{0.3\textwidth}
    \centering
    \begin{threeparttable}
        \resizebox{\linewidth}{!}{
            \begin{tabular}{lcccc}
                \toprule
                \multirow{2}{*}{Method} & \multicolumn{4}{c}{Pearson$(\hat{r},\hat{s})$} \\
                \cmidrule(lr){2-5}
                & MP & HS2 & HS3 & PLLM \\
                \midrule
                Gaussian    & 0.291 & \textbf{0.859} & \textbf{0.592} & \textbf{0.787} \\
                Two-anchor-Skywork  & \textbf{0.412} & 0.041 & -0.252 & 0.356 \\
                Two-anchor-DeepSeek-Pro   & -0.116 & -0.553 & -0.137 & -0.454 \\
                Two-anchor-DeepSeek-Flash & 0.158 & -0.110 & 0.006 & -0.061 \\
                \bottomrule
            \end{tabular}
        }
        \begin{tablenotes}[flushleft]
            \fontsize{5}{6}\selectfont \item \textit{Notes:} MP, HS2, HS3, PLLM denote MultiPref-8K, \\ HelpSteer2-9K,
            HelpSteer3-20K, PersonalLLM-40K.
        \end{tablenotes}
    \end{threeparttable}
    \caption{}
    \label{tab:corr}
\end{subfigure}

    \caption{
        \textbf{Uncertainty representation.}
        \textbf{(a)} BT vs BT(Hard) reward margin.
        \textbf{(b)} Estimated $\hat{r}$ versus $\hat{s}$  for Gaussian trained on \texttt{HelpSteer2}.
        \textbf{(c)} Estimated $\hat{r}$ versus $\hat{s}$ for Two-anchor-DeepSeek-Flash trained on \texttt{HelpSteer2}.
        \textbf{(d)} Pearson correlation between estimated $\hat{r}$ and $\hat{s}$. Results are all based on the Llama-3.2-3B backbone.
    }
    \label{fig:uncertainty_identifiability}
    \vspace{-10pt}
\end{figure}

\begin{figure}[t]
    \centering
    \includegraphics[width=0.93\linewidth]{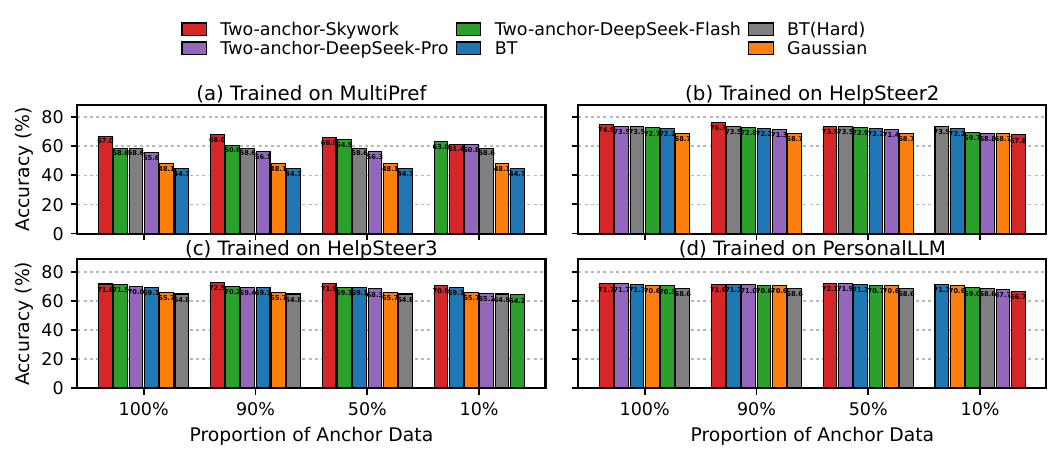}
    \caption{
        \textbf{Anchor data efficiency on RewardBench Accuracy.} Results are all based on the Llama-3.2-1B backbone.
    }
    \label{fig:anchor_efficiency}
    \vspace{-10pt}
\end{figure}

\subsection{Evaluation on reward modeling}

Following prior work, we train reward models on four datasets and then evaluate the learned models under two reward benchmarks, RewardBench \citep{lambert2025rewardbench} and PPE-P \citep{frick2025how}. We report pairwise accuracy and calibration metrics, including Cross-Entropy and Brier score.

\textbf{Main results.} Table~\ref{tab:rm_eval_3b} shows that Two-anchor variants achieve superior performance, not only achieving higher accuracy but also yielding better calibration metrics compared to other baselines. The improvements are consistent when using anchor labels produced by different strategies, including Skywork and DeepSeek-based variants. Meanwhile, Gaussian and BT demonstrate comparable performance, which
supports our motivation that variance modeling requires additional identification constraints.  Notably, BT(Hard) tends to degrade calibration metrics despite moderate accuracy gains. Full results across all backbones are provided in Appendix~\ref{app:experiment_results}

\textbf{Uncertainty representation.}
Figure~\ref{fig:uncertainty_identifiability} illustrates how preference uncertainty is represented by different models. As shown in Figure~\ref{fig:uncertainty_identifiability}(\subref{fig:margin}), BT trained on soft label produces smaller reward margins than BT(Hard), which confirms that annotator disagreement is represented by shrinking the reward margin. The Gaussian model introduces a variance head but $\hat r$ and $\hat s$ remain strongly correlated without anchor supervision (Figures~\ref{fig:uncertainty_identifiability}(\subref{fig:mv_corr})). In contrast, Two-anchor variants weaken this dependence, as shown in Figures~\ref{fig:uncertainty_identifiability}(\subref{fig:mvanchor_corr}) and Table~\ref{fig:uncertainty_identifiability}(\subref{tab:corr}). These results demonstrate that even coarse anchor labels provide sufficient identification signal to disentangle reward quality from preference uncertainty.

\textbf{Anchor data efficiency.} Figure~\ref{fig:anchor_efficiency} illustrates the effect of varying anchor data availability (from 10\% to 100\%) on RewardBench accuracy. Even with only 50\% of anchor data, Two-anchor variants still outperform nearly all baselines with little performance drop, demonstrating that our method is robust to limited anchor data. Such stability effectively reduces the cost of anchor data collection in practice.

\subsection{Evaluation on RLHF}

We evaluate the downstream performance of reward models through both
PPO training and best-of-$N$ (BoN) evaluation. Following prior work \citep{gao2023scaling,coste2023reward,yang2024regularizing,zhang2026bradleyterry}, we select \texttt{Llama-3.2-1B-Instruct} as the base policy model and \texttt{UltraRM-13b}, a strong open-source reward model trained on \texttt{UltraFeedback} with reported state-of-the-art performance among open-source reward models, as the gold reward model. The learned reward model based on the \texttt{Llama-3.2-3B-Instruct} backbone serves as the proxy reward model. In PPO experiment, we downsample 20K samples from the \texttt{Unified-Feedback} dataset to optimize the policy, reserving an additional 500 samples as a held-out set for evaluation. The optimized policy then generates responses for the held-out prompts, which are evaluated by the gold reward model.
For BoN, we apply the same 500 held-out prompts, selecting the best of n responses per prompt based on proxy reward models. The selected responses are then evaluated by the gold reward model. For both PPO and BoN, we average gold scores across all prompts to assess true quality.

Figure~\ref{fig:rlhf} presents PPO and BoN results on the \texttt{UltraFeedback} dataset. As shown in Figure~\ref{fig:rlhf}(\subref{fig:rlhf_a}), Two-anchor variants consistently outperform all baselines in BoN, with Two-anchor-Skywork achieving the best performance. We further evaluate the effect of quantile-based reward, which is defined as $\hat r_q(u)
=
\hat r(u)+\Phi^{-1}(q)\widehat s(u)$. Figure~\ref{fig:rlhf}(\subref{fig:rlhf_b}) shows that lower quantile rewards yield better mean gold scores. This suggests that penalizing response uncertainty leads to more reliable selection. As shown in Figure~\ref{fig:rlhf}(\subref{fig:rlhf_c}), Two-anchor variants maintain stable and strong performance throughout PPO training. Notably, BT, BT(Hard) and Gaussian perform similarly in BoN but BT(Hard) and Gaussian quickly degrade during PPO training due to reward hacking.

\begin{figure}[t]
    \centering

    \begin{subfigure}[b]{0.32\textwidth}
        \centering
        \includegraphics[width=\linewidth]{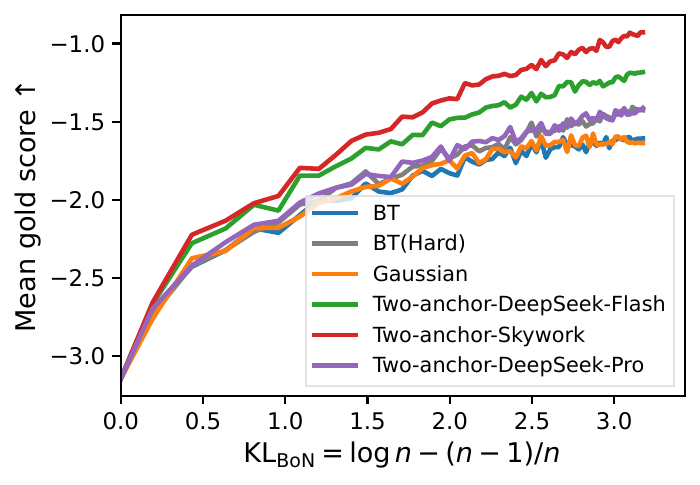}
        \caption{}
        \label{fig:rlhf_a}
    \end{subfigure}
    \hfill
    \begin{subfigure}[b]{0.32\textwidth}
        \centering
        \includegraphics[width=\linewidth]{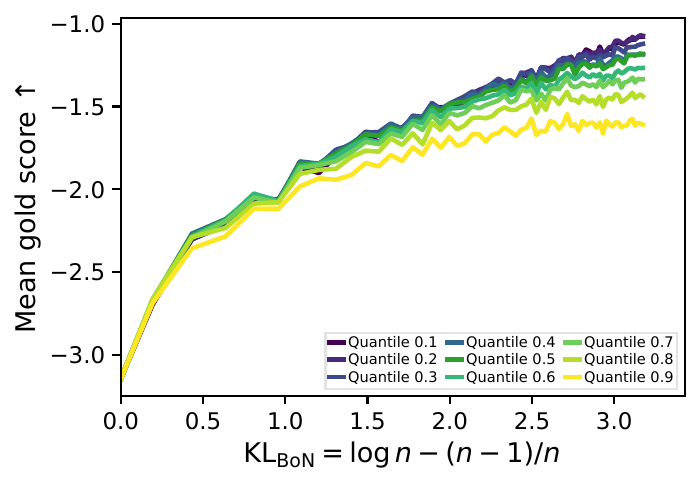}
        \caption{}
        \label{fig:rlhf_b}
    \end{subfigure}
    \hfill
    \begin{subfigure}[b]{0.32\textwidth}
        \centering
        \includegraphics[width=\linewidth]{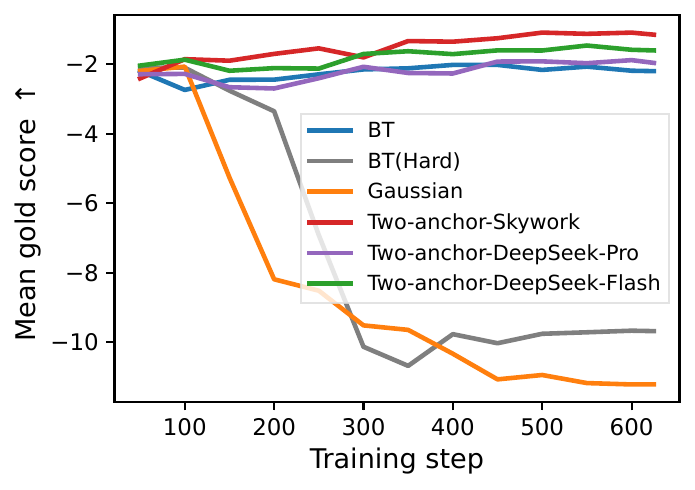}
        \caption{}
        \label{fig:rlhf_c}
    \end{subfigure}
    \hfill

    \caption{
        \textbf{Mean gold score of BoN and PPO.}
        \textbf{(a)} BoN performance VS baselines.
        \textbf{(b)} BoN performance of Two-anchor under different reward distribution quantiles.
        \textbf{(c)} PPO performance VS baselines.  Proxy reward models are all trained on the \texttt{MultiPref}.
    }
    \label{fig:rlhf}
    \vspace{-10pt}
\end{figure}

\section{Conclusion}

In this paper, we propose an effective framework for identifiable variance-aware reward modeling under pluralistic human preferences. On the theoretical side, the two anchors induce a global quadratic curvature of the population loss that drives the estimator to the classical nonparametric minimax rate. On the empirical side,  our method decouples the learned reward mean from the learned variance, remains effective with only 50\% of the anchor data and enhances the Gaussian model against reward hacking during PPO training. These findings suggest that two coarse anchors label are not only auxiliary supervision, but also provide a practical and effective way to make Gaussian reward models identifiable.

\bibliographystyle{plainnat}
\bibliography{staix_2026_sample}
\newpage
\appendix

\section{Addition auxiliary results for the Theory Section}
\label{app:additional}

This appendix collects the auxiliary results supporting the proof of Theorem~\ref{thm:main-rate}. Proposition~\ref{prop:1} ensures that all true preference and anchor probabilities are bounded away from $0$ and $1$. Proposition~\ref{prop:two-anchor-nondegeneracy} characterizes the nondegeneracy of the two-anchor map and provides the geometric intuition behind the curvature in Proposition~\ref{prop:anchor-global-curvature}. Corollary~\ref{cor:joint-curvature} extends the anchor curvature to the joint population loss, while Lemma~\ref{lem:pop-smoothness} provides a quadratic upper bound under  Assumption~\ref{ass:data}. Lemma~\ref{lem:empirical-process} bounds the empirical-process term uniformly over the truncated sieve space. Combining Corollary~\ref{cor:joint-curvature}, Lemma~\ref{lem:empirical-process} and the sieve approximation in Assumption~\ref{ass:2}, we obtain Theorem~\ref{thm:main-rate}. All proofs are deferred to Appendix~\ref{app:proofs}.

\begin{assumption*}[Full version of Assumption~\ref{ass:data}]
\label{full_ass:data}
Let $\mathbb P_\psi$ denote a marginal distribution of the feature representation $\psi(x,y)$. The anchor and preference datasets are generated as follows. \emph{(i) Anchor.} For each anchor sample $(x,y)$, the feature
$\psi(x,y)\sim\mathbb P_\psi$. \emph{(ii) Preference.} For each preference sample $(x,y_1,y_2)$,
the two responses are conditionally i.i.d. given $x$,
$y_1,y_2\mid x \stackrel{}{\sim}\pi(\cdot\mid x)$. The feature representation $u_1=\psi(x,y_1)$ and $u_2=\psi(x,y_2)$ satisfy $u_1\sim\mathbb P_\psi$ and $u_2\sim\mathbb P_\psi$.
\end{assumption*}

\begin{assumption*}[Full version of Assumption~\ref{ass:1}]
    The fixed feature mapping $\psi(x, y) : \mathcal{X} \times \mathcal{Y} \to \mathcal{U} \subset [-\psi_{\max}, \psi_{\max}]^d$ admits a density bounded above and below by positive constants. The true mean and log-standard-deviation functions satisfy $r^* \in \mathcal{H}(\alpha_r, L_r; \mathcal{U})$ and $g^* \in \mathcal{H}(\alpha_g, L_g; \mathcal{U})$ for some H\"older smoothness indices $\alpha_r, \alpha_g > 0$ and constants $L_r, L_g > 0$. Moreover, there exists a positive constant $B$ such that $\|r^*\|_\infty \le B$ and $\|g^*\|_\infty \le B$.
    \label{full_ass:1}
\end{assumption*}

\begin{assumption*}[Full version of Assumption~\ref{ass:2}]
\label{full_ass:2}
The sieve spaces $\mathcal R_K$ and $\mathcal G_K$ satisfy: \emph{(i) Approximation.}
There exists a constant $C>0$, independent of $K$, such that $\inf_{(r,g)\in\Theta_K(B_\Theta)}
\|r-r^*\|_{L_2(\mathbb{P}_\psi)}\le CK^{-\alpha_r/d}$ and $\inf_{(r,g)\in\Theta_K(B_\Theta)}
\|g-g^*\|_{L_2(\mathbb{P}_\psi)}\le CK^{-\alpha_g/d}$; \emph{(ii) Basis functions stability.}
Each $r\in\mathcal R_K$ and $g\in\mathcal G_K$ admits a basis expansion $r(u)=\beta_r^\top B_{r,K}(u)$, $g(u)=\beta_g^\top B_{g,K}(u)$, and there exist constants $0<c_B<C_B<\infty$, independent of $K$, satisfying  $c_B\|\beta_r-\beta_r'\|_2\le\|r-r'\|_{L_2(\mathbb{P}_\psi)}\le C_B\|\beta_r-\beta_r'\|_2$ and $c_B\|\beta_g-\beta_g'\|_2\le\|g-g'\|_{L_2(\mathbb{P}_\psi)}\le C_B\|\beta_g-\beta_g'\|_2$ for all $r,r'\in\mathcal R_K$ and $g,g'\in\mathcal G_K$.
\end{assumption*}

\begin{proposition}[Overlap boundedness]
\label{prop:1}
Under Assumption~\ref{ass:1}, there exists a constant
$c_0\in(0,1/3)$, depending only on $B$ and $\tau_1,\tau_2$, such that
for all $u_1,u_2,u\in\mathcal{U}$, $c_0
\;\le\;
\Phi\!\left(
\frac{r^*(u_1)-r^*(u_2)}
{\sqrt{e^{2g^*(u_1)}+e^{2g^*(u_2)}}}
\right)
\;\le\;
1-c_0$. Moreover, let true anchor probability $q_k^*(u)
:=
\Phi\!\left(
\frac{r^*(u)-\tau_k}{e^{g^*(u)}}
\right)$ for $k\in\{1,2\}$, and define the three threshold-induced anchor classes class probabilities $\pi_0^*(u)=1-q_1^*(u)$, $\pi_1^*(u)=q_1^*(u)-q_2^*(u)$ and $\pi_2^*(u)=q_2^*(u)$. Then $\pi_\ell^*(u)\ge c_0$ for all $\ell\in\{0,1,2\}$
\end{proposition}

Proposition~\ref{prop:1} ensures that the true preference probability is
bounded away from \(0\) and \(1\), and that each true anchor class has a
uniformly positive probability.

\begin{proposition}[Two-anchor map nondegeneracy on bounded sets]
\label{prop:two-anchor-nondegeneracy}
For each $u\in\mathcal{U}$, define the two-anchor map on pointwise values $(r(u),g(u))\in\mathbb{R}^2$ by
\[
\Gamma_u(r(u),g(u))
=
\left(
\Phi\!\left(\tfrac{r(u)-\tau_1}{e^{g(u)}}\right),\
\Phi\!\left(\tfrac{r(u)-\tau_2}{e^{g(u)}}\right)
\right),
\]
and let $J_{\Gamma_u}(r(u),g(u))$ denote its Jacobian. The Jacobian satisfies
\[
\det J_{\Gamma_u}(r(u),g(u))
=
\varphi(z_1)\,\varphi(z_2)\,e^{-2g(u)}\,(\tau_2-\tau_1),
\qquad
z_k:=(r(u)-\tau_k)e^{-g(u)},
\]
where $\varphi$ is the standard normal density. Moreover, for every $R>0$, there exists a constant $c_J=c_J(R,\tau_1,\tau_2)>0$, independent of $u$, such that
\[
J_{\Gamma_u}(r(u),g(u))^\top J_{\Gamma_u}(r(u),g(u))\succeq c_J\,I_2
\qquad\text{for all } u\in\mathcal{U}\text{ and }(r(u),g(u))\in[-R,R]^2.
\]
\end{proposition}

\begin{remark}[Effect of anchor thresholds configuration]
\label{rmk:anchor-placement}
The constant $c_{\mathrm{anc}}$ in Proposition~\ref{prop:anchor-global-curvature} depends on the configuration of the two anchor thresholds $(\tau_1,\tau_2)$ through two competing effects: a gap effect from $\tau_2-\tau_1$ and a tail effect from $\max\{|\tau_1|,|\tau_2|\}$. Notably, (i) The anchor gap $\tau_2-\tau_1$ affects the curvature constant $c_{\mathrm{anc}}$ and two-anchor map’s Jacobian (see Proposition~\ref{prop:two-anchor-nondegeneracy} in Appendix). A nonzero gap is necessary to ensure identifiability. if the gap is too small, the two anchor probabilities become nearly redundant and the curvature weakens; (ii) Extreme thresholds push the anchor probabilities $\Phi((r^*(u)-\tau_k)/e^{g^*(u)})$ toward $0$ or $1$, reducing the information carried by the anchor labels and again weakening the curvature.
\end{remark}

\begin{corollary}[Curvature lower bound for the joint population loss]
\label{cor:joint-curvature}
We denote the population preference loss by
\[
\mathcal{M}_{\mathrm{pref}}(r,g)
=
-\,\mathbb{E}_{(u_1,u_2)\sim\mathbb{P}_\psi^{\mathrm{pref}}}\!\left[
p^*(u_1,u_2)\log p(u_1,u_2)
+
\bigl(1-p^*(u_1,u_2)\bigr)\log\bigl(1-p(u_1,u_2)\bigr)
\right],
\]
where $p(u_1,u_2)
=
\Phi\!\left(
\frac{r(u_1)-r(u_2)}
{\sqrt{e^{2g(u_1)}+e^{2g(u_2)}}}
\right)$ and $p^*(u_1,u_2)$ is defined similarly. The population anchor loss $\mathcal{M}_{\mathrm{anchor}}$
was defined in Proposition~\ref{prop:anchor-global-curvature}. We define
the joint population loss as
\[
\mathcal{M}(r,g)
:=
\mathcal{M}_{\mathrm{pref}}(r,g)
+
\lambda\,\mathcal{M}_{\mathrm{anchor}}(r,g).
\]
Under Assumptions~\ref{ass:1}--\ref{ass:2}, there exists a constant
$c_{\mathrm{pop}}>0$, depending only on
$B_\Theta,\tau_1,\tau_2,\lambda$, such that for all
$(r,g)\in\Theta_K(B_\Theta)$,
\[
\mathcal{M}(r,g)
-
\mathcal{M}(r^*,g^*)
\;\ge\;
c_{\mathrm{pop}}
\left(
\|r-r^*\|_{L_2(\mathbb{P}_\psi)}^2
+
\|g-g^*\|_{L_2(\mathbb{P}_\psi)}^2
\right).
\]
\end{corollary}

\begin{lemma}[Quadratic upper bound on the population excess loss]
\label{lem:pop-smoothness}
Under Assumptions~\ref{ass:data},\ref{ass:1} and \ref{ass:2}, there exists a constant
$C_{\mathrm{pop}}>0$, depending only on
$B_\Theta,\tau_1,\tau_2,\lambda$, such that for all
$(r,g)\in\Theta_K(B_\Theta)$,
\[
\mathcal{M}(r,g)
-
\mathcal{M}(r^*,g^*)
\;\le\;
C_{\mathrm{pop}}
\left(
\|r-r^*\|_{L_2(\mathbb{P}_\psi)}^2
+
\|g-g^*\|_{L_2(\mathbb{P}_\psi)}^2
\right).
\]
\end{lemma}

\begin{lemma}[Empirical-process bound on the truncated sieve]
\label{lem:empirical-process}
We denote
\[
R(r,g)
:=
\|r-r^*\|_{L_2(\mathbb{P}_\psi)}
+
\|g-g^*\|_{L_2(\mathbb{P}_\psi)}
\]
for the joint $L_2$ error, and
\[
G(r,g)
:=
\bigl[
\mathcal{L}(r,g)-\mathcal{L}(r^*,g^*)
\bigr]
-
\bigl[
\mathcal{M}(r,g)-\mathcal{M}(r^*,g^*)
\bigr]
\]
for the centered empirical fluctuation between the empirical and population excess losses. Under Assumptions~\ref{ass:data}, \ref{ass:1} and \ref{ass:2}, for every $t\ge 1$, with probability at least $1-2e^{-t}$, for all $(r,g)\in\Theta_K(B_\Theta)$,
\[
\bigl|G(r,g)\bigr|
\le
C\bigl\{R(r,g)\lambda\Delta_t+\lambda\Delta_t^2\bigr\}.
\]
where $\Delta_t:=\lambda^{-1}\sqrt{(K+t)/n_{\mathrm{pref}}}+\sqrt{(K+t)/n_{\mathrm{anc}}}$. The constant $C>0$ depends only on
$B_\Theta,\tau_1,\tau_2$, and is independent of
$K,n_{\mathrm{pref}},n_{\mathrm{anc}}$, and $\lambda$.
\end{lemma}

\section{Proofs for the Theory Section}
\label{app:proofs}

\subsection{Proof of Proposition~\ref{prop:1}}

\begin{proof}[Proof of Proposition 1]
For notational simplicity, we define the true preference margin and the true anchor threshold scores as $z_{\mathrm{pref}}^*(u_1,u_2)
=
\frac{r^*(u_1)-r^*(u_2)}
{\sqrt{e^{2g^*(u_1)}+e^{2g^*(u_2)}}}$ and $z_{\mathrm{anc},k}^*(u)
:=
\frac{r^*(u)-\tau_k}{e^{g^*(u)}}$ for $k\in\{1,2\}$.

By Assumption~\ref{ass:1}, we have $|r^*(u)|\le B$ and
$|g^*(u)|\le B$ for all $u\in\mathcal U$. It follows that
$|r^*(u_1)-r^*(u_2)|\le 2B$ and $\sqrt{e^{2g^*(u_1)}+e^{2g^*(u_2)}}
\ge
\sqrt{2}\,e^{-B}$.

Therefore, combining these two bounds yields
\begin{equation}
    |z_{\mathrm{pref}}^*(u_1,u_2)|
    \;\le\;
    \frac{2B}{\sqrt{2}\,e^{-B}}
    \;=\;
    \sqrt{2}\,Be^{B}.
    \label{eq:A.1}
    \tag{A.1}
\end{equation}

Again by Assumption~\ref{ass:1}, we obtain
\begin{equation}
|z_{\mathrm{anc},k}^*(u)|
\le
e^B(B+|\tau_k|),
\qquad k\in\{1,2\}.
\label{eq:A.2}
\tag{A.2}
\end{equation}

Then \eqref{eq:A.1} and \eqref{eq:A.2} imply the boundedness of the
true preference margin and the true anchor threshold scores:
\[
|z_{\mathrm{pref}}^*(u_1,u_2)|\le M,
\qquad
|z_{\mathrm{anc},k}^*(u)|\le M,
\quad k\in\{1,2\},
\]
where $M
=
\max\!\left\{
\sqrt{2}\,Be^B,\,
e^B(B+|\tau_1|),\,
e^B(B+|\tau_2|)
\right\}$.

Since the standard normal cumulative distribution function $\Phi$ is strictly increasing, it follows that $\Phi(-M)
\le
\Phi\!\left(z_{\mathrm{pref}}^*(u_1,u_2)\right)
\le
\Phi(M)
=
1-\Phi(-M).$ This proves the boundedness of the true preference margin.

We next prove the lower bounds for the three threshold-induced anchor class probabilities. By the boundedness of $z_{\mathrm{anc},k}^*(u)$, we have the boundedness for true anchor probability
\[
\Phi(-M)
\le
q_k^*(u)
\le
1-\Phi(-M),
\qquad k\in\{1,2\}.
\]
Therefore, we have
\[
\pi_0^*(u)=1-q_1^*(u)\ge \Phi(-M),
\qquad
\pi_2^*(u)=q_2^*(u)\ge \Phi(-M).
\]

For the middle class probability, since $\tau_1<\tau_2$, we have
\[
z_{\mathrm{anc},1}^*(u)-z_{\mathrm{anc},2}^*(u)
=
\frac{\tau_2-\tau_1}{e^{g^*(u)}}
\ge
(\tau_2-\tau_1)e^{-B}.
\]
By the mean value theorem and
$z_{\mathrm{anc},1}^*(u),z_{\mathrm{anc},2}^*(u)\in[-M,M]$, we obtain
\[
\pi_1^*(u)
=
q_1^*(u)-q_2^*(u)
=
\Phi\!\left(z_{\mathrm{anc},1}^*(u)\right)
-
\Phi\!\left(z_{\mathrm{anc},2}^*(u)\right)
\ge
\varphi(M)(\tau_2-\tau_1)e^{-B},
\]
where $\varphi$ is the standard normal density.

Thus, we have $c_0
\le
\Phi\!\left(z_{\mathrm{pref}}^*(u_1,u_2)\right)
\le
1-c_0$ and  $\pi_\ell^*(u)\ge c_0$ for $\ell\in\{0,1,2\}$ by setting $c_0
:=
\frac12
\min\!\left\{
\frac13,\,
\Phi(-M),\,
\varphi(M)(\tau_2-\tau_1)e^{-B}
\right\}
\in(0,1/3)$
This completes the proof.
\label{proof:Proposition 1}
\end{proof}

\subsection{Proof of Proposition~\ref{prop:anchor-global-curvature}}

\begin{proof}[Proof of Global quadratic curvature of the two-anchor population loss]
We define the population anchor loss $\mathcal M_{\mathrm{anchor}}(r,g)
=
\mathbb{E}_{u\sim\mathbb{P}_\psi}\!\left[
-\sum_{\ell=0}^2
\pi_\ell^*(u)\log \pi_\ell(u)
\right]$, where $q_k(u):=\Phi\!\left(\frac{r(u)-\tau_k}{e^{g(u)}}\right)$ and $q_k^*(u):=\Phi\!\left(\frac{r^*(u)-\tau_k}{e^{g^*(u)}}\right)$ for $k\in\{1,2\}$ and the three threshold-induced anchor class probabilities are defined as $\pi_0(u)=1-q_1(u)$, $\pi_1(u)=q_1(u)-q_2(u)$, and $\pi_2(u)=q_2(u)$ with $\pi_\ell^*(u)$ defined similarly. For $k=1,2$ and $u\in\mathcal{U}$, we define $z_{\mathrm{anc},k}(u)
=
\frac{r(u)-\tau_k}{e^{g(u)}}
=
(r(u)-\tau_k)e^{-g(u)}$, $z_{\mathrm{anc},k}^*(u)
=
\frac{r^*(u)-\tau_k}{e^{g^*(u)}}$ and write $q_k(u)=\Phi(z_{\mathrm{anc},k}(u))$ and
$q_k^*(u)=\Phi(z_{\mathrm{anc},k}^*(u))$. Let
$\tau_{\max}:=\max\{|\tau_1|,|\tau_2|\}$ and define $M:=e^{B_\Theta}(B_\Theta+\tau_{\max})$ and $M^*:=e^{B}(B+\tau_{\max})$.

By Assumption~\ref{ass:1} and the truncation $(r,g)\in\Theta_K(B_\Theta)$, we have
$|r(u)|, |g(u)|\le B_\Theta$ and $|r^*(u)|, |g^*(u)|\le B$. Following the same strategy as used in the proof of Proposition~\ref{prop:1}, we obtain
\[
|z_{\mathrm{anc},k}(u)|\le M,\qquad
|z_{\mathrm{anc},k}^*(u)|\le M^*,
\qquad k\in\{1,2\}.
\tag{A.3}
\label{eq:zbound}
\]

We first establish the Lipschitz bound from $q_k$ to $z_{\mathrm{anc},k}$. Combining \eqref{eq:zbound} and the strictly increasing property of $\Phi$, we have
\[
\Phi(-M)\le q_k(u),q_k^*(u)\le \Phi(M).
\]
Since $\Phi^{-1}$ is continuously differentiable with derivative
$1/\varphi(\Phi^{-1}(\cdot))$ and attains its maximum at the endpoints, $\Phi^{-1}$ is $L_\Phi$-Lipschitz on
$[\Phi(-M),\Phi(M)]$ with $L_\Phi=\frac{1}{\varphi(M)}$. It follows that, for every $u\in\mathcal{U}$ and $k=1,2$,
\[
|z_{\mathrm{anc},k}(u)-z_{\mathrm{anc},k}^*(u)|
\le
L_\Phi |q_k(u)-q_k^*(u)|.
\tag{A.4}
\label{eq:phi-lip}
\]

We next establish a pointwise Lipschitz bound from
$(z_{\mathrm{anc},1}(u),z_{\mathrm{anc},2}(u))$ to
$(r(u),g(u))$. For notational simplicity, $z_1(u)=z_{\mathrm{anc},1}(u)$ and
$z_2(u)=z_{\mathrm{anc},2}(u)$ in the following calculation. Since
$\tau_2>\tau_1$, for every $u\in\mathcal U$,
$z_1(u)-z_2(u)=(\tau_2-\tau_1)e^{-g(u)}>0$. Moreover, $r(u)$ and $g(u)$ follow the map
in terms of $z_1(u)$ and $z_2(u)$:
\[
r(u)=\frac{\tau_2 z_1(u)-\tau_1 z_2(u)}{z_1(u)-z_2(u)},
\qquad
g(u)=\log\frac{\tau_2-\tau_1}{z_1(u)-z_2(u)}.
\tag{A.5}
\label{eq:inv}
\]

The bounds $|g|\le B_\Theta$ and $|z_{\mathrm{anc},k}|\le M$ imply that the corresponding
pairs $(z_{\mathrm{anc},1}(u),z_{\mathrm{anc},2}(u))$ and
$(z_{\mathrm{anc},1}^*(u),z_{\mathrm{anc},2}^*(u))$ both belong to the compact
convex set
\[
D:=\Bigl\{(z_1(u),z_2(u))\in\mathbb{R}^2:
(\tau_2-\tau_1)e^{-B_\Theta}\le z_1(u)-z_2(u)\le
(\tau_2-\tau_1)e^{B_\Theta}
\Bigr\},
\tag{A.6}
\label{eq:convex_set}
\]
where $\
\max\{|z_1(u)|,|z_2(u)|\}\le M$.

The Jacobian of the map in~\eqref{eq:inv} is
\[
J_{\mathrm{inv}}(z_1(u),z_2(u))
=
\begin{pmatrix}
-\dfrac{(\tau_2-\tau_1)z_2(u)}{(z_1(u)-z_2(u))^2}
&
\dfrac{(\tau_2-\tau_1)z_1(u)}{(z_1(u)-z_2(u))^2}
\\[10pt]
-\dfrac{1}{z_1(u)-z_2(u)}
&
\dfrac{1}{z_1(u)-z_2(u)}
\end{pmatrix}.
\]

Based on the convex set $D$ \eqref{eq:convex_set}, we have
$z_1(u)-z_2(u)\ge \delta_\Theta>0$ and
$\max\{|z_1(u)|,|z_2(u)|\}\le M$, where
$\delta_\Theta :=(\tau_2-\tau_1)e^{-B_\Theta}$. Therefore,
$\sup_{(z_1(u),z_2(u))\in D}\|J_{\mathrm{inv}}(z_1(u),z_2(u))\|_F\le L_z$, where
\[
\begin{aligned}
\|J_{\mathrm{inv}}(z_1(u),z_2(u))\|_F^2
&=
\left(
\frac{(\tau_2-\tau_1)z_2(u)}{(z_1(u)-z_2(u))^2}
\right)^2
+
\left(
\frac{(\tau_2-\tau_1)z_1(u)}{(z_1(u)-z_2(u))^2}
\right)^2
+ \\ &
\left(
\frac{1}{z_1(u)-z_2(u)}
\right)^2
+
\left(
\frac{1}{z_1(u)-z_2(u)}
\right)^2
\\
&\le
2\left(
\frac{(\tau_2-\tau_1)M}{\delta_\Theta^2}
\right)^2
+
2\left(
\frac{1}{\delta_\Theta}
\right)^2
\\
&=:
L_z^2,
\end{aligned}
\]
and $L_z=L_z(B_\Theta,\tau_1,\tau_2)<\infty$.

Given the convex set $D$ defined in \eqref{eq:convex_set}, applying the mean value inequality to the map defined in \eqref{eq:inv} yields that, for every $u\in\mathcal U$,
\[
\begin{aligned}
|r(u)-r^*(u)|^2+|g(u)-g^*(u)|^2
&=
\left\|
\begin{pmatrix}
r(u)\\
g(u)
\end{pmatrix}
-
\begin{pmatrix}
r^*(u)\\
g^*(u)
\end{pmatrix}
\right\|_2^2
\\
&\le
L_z^2
\left\|
\begin{pmatrix}
z_{1}(u)\\
z_{2}(u)
\end{pmatrix}
-
\begin{pmatrix}
z_{1}^*(u)\\
z_{2}^*(u)
\end{pmatrix}
\right\|_2^2
\\
&=
L_z^2
\left[
|z_{1}(u)-z_{1}^*(u)|^2
+
|z_{2}(u)-z_{2}^*(u)|^2
\right].
\end{aligned}
\tag{A.7}
\label{eq:zr-bound}
\]

Combining the two Lipschitz bounds \eqref{eq:phi-lip} and \eqref{eq:zr-bound}, we obtain
\[
|r(u)-r^*(u)|^2+|g(u)-g^*(u)|^2
\le
L_z^2 L_\Phi^2
\bigl[
|q_1(u)-q_1^*(u)|^2+
|q_2(u)-q_2^*(u)|^2
\bigr].
\tag{A.8}
\label{eq:rg-q}
\]

For the anchor part, conditional on $u$, the pointwise population loss difference equals $\mathrm{KL}\!\left(
\pi^*(u)
\,\middle\|\,
\pi(u)
\right)$, where $\pi(u)=(\pi_0(u),\pi_1(u),\pi_2(u))$ and $\pi^*(u)=(\pi_0^*(u),\pi_1^*(u),\pi_2^*(u))$.

For $k\in\{1,2\}$, we define $\delta_k(u):=q_k(u)-q_k^*(u)$. By the definitions of the three threshold-induced anchor class probabilities, we have $\pi_0(u)-\pi_0^*(u)=-\delta_1(u)$, $\pi_1(u)-\pi_1^*(u)=\delta_1(u)-\delta_2(u)$ and $\pi_2(u)-\pi_2^*(u)=\delta_2(u)$.
Therefore, we obtain
\[
\begin{aligned}
\|\pi(u)-\pi^*(u)\|_2^2
&=
|\delta_1(u)|^2+
|\delta_1(u)-\delta_2(u)|^2+
|\delta_2(u)|^2
\\
&\ge
|\delta_1(u)|^2+|\delta_2(u)|^2
\\
&=
|q_1(u)-q_1^*(u)|^2+
|q_2(u)-q_2^*(u)|^2.
\end{aligned}
\tag{A.9}
\label{eq:three_lower}
\]

By Pinsker's inequality and \eqref{eq:rg-q}, we have
\[
\begin{aligned}
\mathrm{KL}\!\left(
\pi^*(u)
\,\middle\|\,
\pi(u)
\right)
&\ge
\frac{1}{2}\|\pi(u)-\pi^*(u)\|_1^2
\\
&\ge
\frac{1}{2}\|\pi(u)-\pi^*(u)\|_2^2
\\
&\ge
\frac{1}{2}
\bigl[
|q_1(u)-q_1^*(u)|^2+
|q_2(u)-q_2^*(u)|^2
\bigr]
\\
&\ge
\frac{1}{2L_z^2L_\Phi^2}
\left\{
|r(u)-r^*(u)|^2+
|g(u)-g^*(u)|^2
\right\}.
\end{aligned}
\tag{A.10}
\label{eq:pointwise}
\]

Finally, integrating \eqref{eq:pointwise} with respect to $u\sim \mathbb P_\psi$ gives
\[
\begin{aligned}
\mathcal{M}_{\mathrm{anchor}}(r,g)
-
\mathcal{M}_{\mathrm{anchor}}(r^*,g^*)
&=
\mathbb{E}_{u\sim \mathbb P_\psi}
\left[
\mathrm{KL}\!\left(
\pi^*(u)
\,\middle\|\,
\pi(u)
\right)
\right]
\\
&\ge
c_{\mathrm{anc}}
\left(
\|r-r^*\|_{L_2(\mathbb P_\psi)}^2+
\|g-g^*\|_{L_2(\mathbb P_\psi)}^2
\right),
\end{aligned}
\]
where $c_{\mathrm{anc}}
:=
\frac{1}{2L_z^2L_\Phi^2}>0$.
The constant $c_{\mathrm{anc}}$ depends only on
$B_\Theta,\tau_1,\tau_2$, and is independent of $(r,g)$ and $K$. This completes the proof.
\end{proof}

\subsection{Proof of Corollary~\ref{cor:joint-curvature}}
\begin{proof}
For any $(r,g)\in\Theta_K(B_\Theta)$,
\[
\begin{aligned}
&\mathcal{M}(r,g)
-\mathcal{M}(r^*,g^*)
\\
&\quad =
\Bigl[
\mathcal{M}_{\mathrm{pref}}(r,g)
-
\mathcal{M}_{\mathrm{pref}}(r^*,g^*)
\Bigr]
\\
&\qquad
+
\lambda
\Bigl[
\mathcal{M}_{\mathrm{anchor}}(r,g)
-
\mathcal{M}_{\mathrm{anchor}}(r^*,g^*)
\Bigr].
\end{aligned}
\]
For the preference component, we have
\[
\mathcal{M}_{\mathrm{pref}}(r,g)
-
\mathcal{M}_{\mathrm{pref}}(r^*,g^*)
=
\mathbb{E}_{(u_1,u_2)\sim\mathbb{P}_\psi^{\mathrm{pref}}}
\left[
\mathrm{KL}\!\left(
\mathrm{Bern}(p^*(u_1,u_2))\,\|\,\mathrm{Bern}(p(u_1,u_2))
\right)
\right]
\ge 0.
\tag{A.11}
\label{eq:pref-nonneg}
\]
For the anchor component, Proposition~\ref{prop:anchor-global-curvature}
yields
\[
\begin{aligned}
&\mathcal{M}_{\mathrm{anchor}}(r,g)
-
\mathcal{M}_{\mathrm{anchor}}(r^*,g^*)
\\
&\quad\ge
c_{\mathrm{anc}}
\left(
\|r-r^*\|_{L_2(\mathbb P_\psi)}^2
+
\|g-g^*\|_{L_2(\mathbb P_\psi)}^2
\right).
\end{aligned}
\tag{A.12}
\label{eq:anc-curv}
\]
Combining \eqref{eq:pref-nonneg} and \eqref{eq:anc-curv}, and using
$\lambda>0$, we have
\[
\begin{aligned}
&\mathcal{M}(r,g)
-
\mathcal{M}(r^*,g^*)
\\
&\quad\ge
c_{\mathrm{pop}}
\left(
\|r-r^*\|_{L_2(\mathbb P_\psi)}^2
+
\|g-g^*\|_{L_2(\mathbb P_\psi)}^2
\right),
\end{aligned}
\tag{A.13}
\label{eq:c_pop}
\]
where $c_{\mathrm{pop}}:=\lambda c_{\mathrm{anc}}>0$.
\end{proof}

\subsection{Proof of Lemma~\ref{lem:pop-smoothness}}

\begin{proof}[Proof of Quadratic upper bound on the population excess loss]
The joint excess loss decomposes into preference and anchor parts:
\[
\begin{aligned}
&\mathcal{M}(r,g)
-\mathcal{M}(r^*,g^*)
\\
&\quad=
\Bigl[
\mathcal{M}_{\mathrm{pref}}(r,g)
-\mathcal{M}_{\mathrm{pref}}(r^*,g^*)
\Bigr]
\\
&\qquad+
\lambda
\Bigl[
\mathcal{M}_{\mathrm{anchor}}(r,g)
-\mathcal{M}_{\mathrm{anchor}}(r^*,g^*)
\Bigr].
\end{aligned}
\tag{A.14}
\label{eq:smoothness-decomp}
\]
We bound each component in \eqref{eq:smoothness-decomp}. We first establish uniform KL upper bounds. Since
$(r,g)\in\Theta_K(B_\Theta)$ and
$\|r^*\|_\infty,\|g^*\|_\infty\le B<B_\Theta$, the same argument as in Proposition~\ref{prop:1} ensures that the preference probabilities
$p(u_1,u_2)$ and $p^*(u_1,u_2)$ are uniformly bounded away from $0$ and
$1$. Therefore, there exists a constant $C_{\mathrm{KL}}>0$, depending
only on $B_\Theta$, such that for every $(u_1,u_2)$,
\[
\mathrm{KL}\!\left(
\mathrm{Bern}(p^*(u_1,u_2))\,\middle\|\,
\mathrm{Bern}(p(u_1,u_2))
\right)
\le
C_{\mathrm{KL}}
\left(p(u_1,u_2)-p^*(u_1,u_2)\right)^2.
\tag{A.15}
\label{eq:KL-upper}
\]

Moreover, the three threshold-induced anchor class probabilities
$\pi_\ell(u)$ and $\pi_\ell^*(u)$, $\ell\in\{0,1,2\}$, are uniformly
bounded away from zero. Hence there exists a constant $C_\pi>0$,
depending only on $B_\Theta,\tau_1,\tau_2$, such that
\[
\mathrm{KL}\!\left(
\pi^*(u)\,\middle\|\,\pi(u)
\right)
\le
C_\pi\|\pi(u)-\pi^*(u)\|_2^2.
\tag{A.16}
\label{eq:multi-KL-upper}
\]

We next establish a pointwise Lipschitz bound from $(r(u),g(u))$ to the
anchor probabilities. As in the proof of
Proposition~\ref{prop:anchor-global-curvature}, define $z_{\mathrm{anc},k}(u)=(r(u)-\tau_k)e^{-g(u)}$ and $z_{\mathrm{anc},k}^*(u)=(r^*(u)-\tau_k)e^{-g^*(u)}$, so that $q_k(u)=\Phi(z_{\mathrm{anc},k}(u))$ and
$q_k^*(u)=\Phi(z_{\mathrm{anc},k}^*(u))$. The pointwise map $(r(u),g(u))
\mapsto
z_{\mathrm{anc},k}(u)
=
(r(u)-\tau_k)e^{-g(u)}$ has gradient
\[
\nabla_{(r(u),g(u))} z_{\mathrm{anc},k}(u)
=
\bigl(e^{-g(u)},\,-(r(u)-\tau_k)e^{-g(u)}\bigr).
\]
Let $\tau_{\max}:=\max\{|\tau_1|,|\tau_2|\}$. Since
$|r(u)|\le B_\Theta$ and $|g(u)|\le B_\Theta$, its Euclidean norm satisfies
\[
\left\|
\nabla_{(r(u),g(u))} z_{\mathrm{anc},k}(u)
\right\|_2
\le
e^{B_\Theta}\sqrt{1+(B_\Theta+\tau_{\max})^2}
=:
L_anc.
\tag{A.17}
\label{eq:zk-grad}
\]

Combining \eqref{eq:zk-grad} with
$\|\Phi'\|_\infty=\varphi(0)\le 1/\sqrt{2\pi}$ and the chain rule yields,
for every $u\in\mathcal U$ and $k\in\{1,2\}$,
\[
\begin{aligned}
|q_k(u)-q_k^*(u)|
&=
|\Phi(z_{\mathrm{anc},k}(u))-\Phi(z_{\mathrm{anc},k}^*(u))|
\\
&\le
\|\Phi'\|_\infty\,|z_{\mathrm{anc},k}(u)-z_{\mathrm{anc},k}^*(u)|
\\
&\le
\frac{1}{\sqrt{2\pi}}
\sup_{(r,g)\in[-B_\Theta,B_\Theta]^2}
\|\nabla_{(r,g)}z_{\mathrm{anc},k}(r,g)\|_2
\cdot
\sqrt{|r(u)-r^*(u)|^2+|g(u)-g^*(u)|^2}
\\
&\le
\frac{L_anc}{\sqrt{2\pi}}
\sqrt{|r(u)-r^*(u)|^2+|g(u)-g^*(u)|^2}.
\end{aligned}
\]
Squaring both sides yields, for every $u\in\mathcal U$ and
$k\in\{1,2\}$,
\[
|q_k(u)-q_k^*(u)|^2
\le
\frac{L_{anc}^2}{2\pi}
\bigl(|r(u)-r^*(u)|^2+|g(u)-g^*(u)|^2\bigr).
\tag{A.18}
\label{eq:q-Lip}
\]

 By the definitions in \eqref{eq:three_lower}, we have
\[
\begin{aligned}
\|\pi(u)-\pi^*(u)\|_2^2
&=
|\delta_1(u)|^2+
|\delta_1(u)-\delta_2(u)|^2+
|\delta_2(u)|^2
\\
&\le
3\left(|\delta_1(u)|^2+|\delta_2(u)|^2\right)
\\
&=
3\left[
|q_1(u)-q_1^*(u)|^2+
|q_2(u)-q_2^*(u)|^2
\right].
\end{aligned}
\tag{A.19}
\label{eq:pi-q-upper}
\]

Combining \eqref{eq:multi-KL-upper}, \eqref{eq:q-Lip}, and
\eqref{eq:pi-q-upper}, and integrating with respect to
$u\sim\mathbb P_\psi$ (Assumption~\ref{ass:data}), we obtain
\[
\begin{aligned}
&\mathcal{M}_{\mathrm{anchor}}(r,g)
-\mathcal{M}_{\mathrm{anchor}}(r^*,g^*)
\\
&\quad=
\mathbb{E}_{u\sim\mathbb P_\psi}
\left[
\mathrm{KL}\!\left(
\pi^*(u)\,\middle\|\,\pi(u)
\right)
\right]
\\
&\quad\le
C_\pi
\mathbb{E}_{u\sim\mathbb P_\psi}
\left[
\|\pi(u)-\pi^*(u)\|_2^2
\right]
\\
&\quad\le
3C_\pi
\sum_{k=1}^{2}
\mathbb{E}_{u\sim\mathbb P_\psi}
\left[
|q_k(u)-q_k^*(u)|^2
\right]
\\
&\quad\le
C_{\mathrm{anc}}
\left(
\|r-r^*\|_{L_2(\mathbb P_\psi)}^2
+
\|g-g^*\|_{L_2(\mathbb P_\psi)}^2
\right),
\end{aligned}
\tag{A.20}
\label{eq:anc-smoothness}
\]
where $C_{\mathrm{anc}}:=3C_\pi L_{anc}^2/\pi$ depends only on
$B_\Theta,\tau_1,\tau_2$.

We now establish the Lipschitz bound from $(r(u),g(u))$ to the preference
probability. Define
\[
z_{\mathrm{pref}}(u_1,u_2)
=
\frac{r(u_1)-r(u_2)}
{\sqrt{e^{2g(u_1)}+e^{2g(u_2)}}},
\qquad
z_{\mathrm{pref}}^*(u_1,u_2)
=
\frac{r^*(u_1)-r^*(u_2)}
{\sqrt{e^{2g^*(u_1)}+e^{2g^*(u_2)}}},
\]
so that $p(u_1,u_2)=\Phi(z_{\mathrm{pref}}(u_1,u_2))$ and
$p^*(u_1,u_2)=\Phi(z_{\mathrm{pref}}^*(u_1,u_2))$.

Write $\theta=(\theta_1,\theta_2,\theta_3,\theta_4):=
\bigl(r(u_1),r(u_2),g(u_1),g(u_2)\bigr)$, so that
$z_{\mathrm{pref}}$ can be viewed as a function of
$\theta\in\mathbb{R}^4$. A direct computation gives
\[
\partial_{\theta_1}z_{\mathrm{pref}}
=
\frac{1}{\sqrt{e^{2\theta_3}+e^{2\theta_4}}},
\qquad
\partial_{\theta_2}z_{\mathrm{pref}}
=
-\frac{1}{\sqrt{e^{2\theta_3}+e^{2\theta_4}}},
\]
\[
\partial_{\theta_3}z_{\mathrm{pref}}
=
-\frac{(\theta_1-\theta_2)e^{2\theta_3}}
{\bigl(e^{2\theta_3}+e^{2\theta_4}\bigr)^{3/2}},
\qquad
\partial_{\theta_4}z_{\mathrm{pref}}
=
-\frac{(\theta_1-\theta_2)e^{2\theta_4}}
{\bigl(e^{2\theta_3}+e^{2\theta_4}\bigr)^{3/2}}.
\]
On $[-B_\Theta,B_\Theta]^4$, we have
$|\theta_1-\theta_2|\le 2B_\Theta$ and
$\sqrt{e^{2\theta_3}+e^{2\theta_4}}\ge\sqrt{2}\,e^{-B_\Theta}$, hence
\[
|\partial_{\theta_1}z_{\mathrm{pref}}|,\,
|\partial_{\theta_2}z_{\mathrm{pref}}|
\le
\frac{e^{B_\Theta}}{\sqrt{2}},
\qquad
|\partial_{\theta_3}z_{\mathrm{pref}}|,\,
|\partial_{\theta_4}z_{\mathrm{pref}}|
\le
\sqrt{2}\,B_\Theta\,e^{4B_\Theta}.
\]
Therefore, the Euclidean norm of the gradient
$\nabla_\theta z_{\mathrm{pref}}$ satisfies, on
$[-B_\Theta,B_\Theta]^4$,
\[
\|\nabla_\theta z_{\mathrm{pref}}\|_2
\le
\sqrt{e^{2B_\Theta}+4B_\Theta^2\,e^{8B_\Theta}}
=:
L_{\mathrm{pref}}.
\tag{A.21}
\label{eq:pref-grad}
\]

Combining \eqref{eq:pref-grad} with
$\|\Phi'\|_\infty=\varphi(0)\le 1/\sqrt{2\pi}$ and the chain rule, for
every $(u_1,u_2)$,
\[
\begin{aligned}
|p(u_1,u_2)-p^*(u_1,u_2)|
&=
|\Phi(z_{\mathrm{pref}}(u_1,u_2))
-\Phi(z_{\mathrm{pref}}^*(u_1,u_2))|
\\
&\le
\|\Phi'\|_\infty
|z_{\mathrm{pref}}(u_1,u_2)-z_{\mathrm{pref}}^*(u_1,u_2)|
\\
&\le
\frac{L_{\mathrm{pref}}}{\sqrt{2\pi}}
\sqrt{
\sum_{a=1}^{2}
\bigl(
|r(u_a)-r^*(u_a)|^2+
|g(u_a)-g^*(u_a)|^2
\bigr)
}.
\end{aligned}
\]
Squaring both sides yields, for every $(u_1,u_2)$,
\[
|p(u_1,u_2)-p^*(u_1,u_2)|^2
\le
\frac{L_{\mathrm{pref}}^2}{2\pi}
\sum_{a=1}^{2}
\bigl(
|r(u_a)-r^*(u_a)|^2+
|g(u_a)-g^*(u_a)|^2
\bigr).
\tag{A.22}
\label{eq:p-Lip}
\]

Combining \eqref{eq:KL-upper} and \eqref{eq:p-Lip}, and taking expectation
under $(u_1,u_2)\sim\mathbb P_\psi^{\mathrm{pref}}$ (Assumption~\ref{ass:data}), we obtain
\[
\begin{aligned}
&\mathcal{M}_{\mathrm{pref}}(r,g)
-\mathcal{M}_{\mathrm{pref}}(r^*,g^*)
\\
&\quad=
\mathbb{E}_{(u_1,u_2)\sim\mathbb P_\psi^{\mathrm{pref}}}
\left[
\mathrm{KL}\!\left(
\mathrm{Bern}(p^*(u_1,u_2))
\,\middle\|\,
\mathrm{Bern}(p(u_1,u_2))
\right)
\right]
\\
&\quad\le
C_{\mathrm{KL}}\,
\mathbb{E}_{(u_1,u_2)\sim\mathbb P_\psi^{\mathrm{pref}}}
\left[
|p(u_1,u_2)-p^*(u_1,u_2)|^2
\right]
\\
&\quad\le
\frac{C_{\mathrm{KL}}L_{\mathrm{pref}}^2}{2\pi}
\mathbb{E}_{(u_1,u_2)\sim\mathbb P_\psi^{\mathrm{pref}}}
\left[
\sum_{a=1}^{2}
\bigl(
|r(u_a)-r^*(u_a)|^2+
|g(u_a)-g^*(u_a)|^2
\bigr)
\right]
\\
&\quad=
\frac{C_{\mathrm{KL}}L_{\mathrm{pref}}^2}{\pi}
\mathbb{E}_{u\sim\mathbb P_\psi}
\left[
|r(u)-r^*(u)|^2+
|g(u)-g^*(u)|^2
\right]
\\
&\quad=
C_{\mathrm{pref,smooth}}
\left(
\|r-r^*\|_{L_2(\mathbb P_\psi)}^2
+
\|g-g^*\|_{L_2(\mathbb P_\psi)}^2
\right),
\end{aligned}
\tag{A.23}
\label{eq:pref-smoothness}
\]
where $C_{\mathrm{pref,smooth}}:=C_{\mathrm{KL}}L_{\mathrm{pref}}^2/\pi$
depends only on $B_\Theta$.

Finally, substituting \eqref{eq:anc-smoothness} and
\eqref{eq:pref-smoothness} into \eqref{eq:smoothness-decomp},
\[
\begin{aligned}
&\mathcal{M}(r,g)
-\mathcal{M}(r^*,g^*)
\\
&\quad\le
C_{\mathrm{pop}}
\left(
\|r-r^*\|_{L_2(\mathbb P_\psi)}^2
+
\|g-g^*\|_{L_2(\mathbb P_\psi)}^2
\right),
\end{aligned}
\tag{A.24}
\label{eq:C_pop}
\]
where $C_{\mathrm{pop}}:=C_{\mathrm{pref,smooth}}+\lambda C_{\mathrm{anc}}>0$.
The constant $C_{\mathrm{pop}}$ depends only on
$B_\Theta,\tau_1,\tau_2,\lambda$, and is independent of $(r,g)$ and $K$.
This completes the proof.
\end{proof}

\subsection{Proof of Lemma~\ref{lem:empirical-process}}
\begin{proof}
We prove the result in several steps.

\paragraph{Step 1: Decompose the empirical fluctuation.}
Let \(P_{\mathrm{pref}}\) and \(P_{\mathrm{anc}}\) denote the population
distributions of the preference observation
\(Z_{\mathrm{pref}}=(u_1,u_2,\bar C)\) and the anchor observation
\(Z_{\mathrm{anc}}=(u,A_1,A_2)\), respectively. Let
\(\mathbb P_{n_{\mathrm{pref}}}\) and \(\mathbb P_{n_{\mathrm{anc}}}\)
denote their empirical measures. For an anchor observation
$Z_{\mathrm{anc}}=(u,A_1,A_2)$, define the pointwise anchor excess loss
\[
h_{\mathrm{anc},r,g}(Z_{\mathrm{anc}})
:=
\ell_{\mathrm{anc}}(r,g;Z_{\mathrm{anc}})
-
\ell_{\mathrm{anc}}(r^*,g^*;Z_{\mathrm{anc}}),
\]
where
\[
\ell_{\mathrm{anc}}(r,g;Z_{\mathrm{anc}})
=
-\Bigl[
(1-A_1)\log \pi_0(u)
+
(A_1-A_2)\log \pi_1(u)
+
A_2\log \pi_2(u)
\Bigr].
\]
For a preference observation $Z_{\mathrm{pref}}=(u_1,u_2,\bar C)$,
define the pointwise preference excess loss
\[
h_{\mathrm{pref},r,g}(Z_{\mathrm{pref}})
:=
\ell_{\mathrm{pref}}(r,g;Z_{\mathrm{pref}})
-
\ell_{\mathrm{pref}}(r^*,g^*;Z_{\mathrm{pref}}),
\]
where
\[
\ell_{\mathrm{pref}}(r,g;Z_{\mathrm{pref}})
=
-\Bigl[
\bar C\log p(u_1,u_2)
+
(1-\bar C)\log\{1-p(u_1,u_2)\}
\Bigr].
\]
Here $\bar C\in[0,1]$ denotes the soft preference label and satisfies
$\mathbb E[\bar C\mid u_1,u_2]=p^*(u_1,u_2)$. Then the centered empirical
fluctuation can be written as
\[
G(r,g)
=
(\mathbb P_{n_{\mathrm{pref}}}-P_{\mathrm{pref}})
h_{\mathrm{pref},r,g}
+
\lambda
(\mathbb P_{n_{\mathrm{anc}}}-P_{\mathrm{anc}})
h_{\mathrm{anc},r,g}.
\tag{A.25}
\label{eq:emp-decomp}
\]

\paragraph{Step 2: Establish the $L_2$ size of the loss differences.}
By the truncation $(r,g)\in\Theta_K(B_\Theta)$ and
Assumption~\ref{ass:1}, we have $\|r\|_\infty,\|g\|_\infty\le B_\Theta$
and $\|r^*\|_\infty,\|g^*\|_\infty\le B<B_\Theta$. The same argument as in
the proof of Proposition~\ref{prop:1} yields that the preference
probabilities $p(u_1,u_2)$ and $p^*(u_1,u_2)$ are uniformly bounded away
from $0$ and $1$. Moreover, the three threshold-induced anchor class
probabilities $\pi_\ell(u)$ and $\pi_\ell^*(u)$, $\ell\in\{0,1,2\}$, are
uniformly bounded away from zero. Therefore, the log-loss functions have
uniformly bounded first derivatives on $\Theta_K(B_\Theta)$.

By the Lipschitz bound in \eqref{eq:q-Lip} and the
definitions $\pi_0(u)=1-q_1(u)$, $\pi_1(u)=q_1(u)-q_2(u)$, and
$\pi_2(u)=q_2(u)$, we have
\[
|h_{\mathrm{anc},r,g}(Z_{\mathrm{anc}})|
\le
C\bigl(|r(u)-r^*(u)|+|g(u)-g^*(u)|\bigr),
\]
where $C$ depends only on $B_\Theta,\tau_1,\tau_2$.

By Assumption~\ref{ass:data}, integrating with respect to
$u\sim\mathbb P_\psi$,
\[
\begin{aligned}
\mathbb{E}_{Z_{\mathrm{anc}}}|h_{\mathrm{anc},r,g}|^2
&\le
2C^2\,\mathbb{E}_{u\sim\mathbb P_\psi}
\bigl[
|r(u)-r^*(u)|^2+|g(u)-g^*(u)|^2
\bigr]
\\
&=
2C^2
\left(
\|r-r^*\|_{L_2(\mathbb P_\psi)}^2
+
\|g-g^*\|_{L_2(\mathbb P_\psi)}^2
\right).
\end{aligned}
\]
Hence,
\[
\|h_{\mathrm{anc},r,g}\|_{L_2(P_{\mathrm{anc}})}
\le
C\,R(r,g).
\tag{A.26}
\label{eq:anc-l2}
\]

For the preference loss, since $\bar C\in[0,1]$ and
$p(u_1,u_2)$ is uniformly bounded away from $0$ and $1$, the soft-label
cross-entropy is uniformly Lipschitz in $p(u_1,u_2)$. The preference
probability $p(u_1,u_2)=\Phi(z_{\mathrm{pref}}(u_1,u_2))$ is Lipschitz in
$(r(u_1),r(u_2),g(u_1),g(u_2))$ on the bounded set
$\Theta_K(B_\Theta)$. Therefore,
\[
\begin{aligned}
|h_{\mathrm{pref},r,g}(Z_{\mathrm{pref}})|
\le C\bigl\{&
|r(u_1)-r^*(u_1)|
+
|r(u_2)-r^*(u_2)|
\\
&+
|g(u_1)-g^*(u_1)|
+
|g(u_2)-g^*(u_2)|
\bigr\}.
\end{aligned}
\]
Since Assumption~\ref{ass:data} implies that both
preference feature marginals satisfy $u_1\sim\mathbb P_\psi$ and
$u_2\sim\mathbb P_\psi$, we obtain
\[
\begin{aligned}
\mathbb{E}_{Z_{\mathrm{pref}}}|h_{\mathrm{pref},r,g}|^2
&\le
4C^2\,
\mathbb{E}_{Z_{\mathrm{pref}}}
\left[
\sum_{a=1}^2
\bigl(
|r(u_a)-r^*(u_a)|^2
+
|g(u_a)-g^*(u_a)|^2
\bigr)
\right]
\\
&=
8C^2
\mathbb{E}_{u\sim\mathbb P_\psi}
\bigl[
|r(u)-r^*(u)|^2
+
|g(u)-g^*(u)|^2
\bigr]
\\
&=
8C^2
\left(
\|r-r^*\|_{L_2(\mathbb P_\psi)}^2
+
\|g-g^*\|_{L_2(\mathbb P_\psi)}^2
\right).
\end{aligned}
\]
Hence,
\[
\|h_{\mathrm{pref},r,g}\|_{L_2(P_{\mathrm{pref}})}
\le
C\,R(r,g).
\tag{A.27}
\label{eq:pref-l2}
\]

\paragraph{Step 3: Localize by $R$-shells.}
For $s>0$, define the shell
\[
\mathcal S_s
=
\Bigl\{
(r,g)\in\Theta_K(B_\Theta):
\tfrac{s}{2}\le R(r,g)\le s
\Bigr\}.
\]
By \eqref{eq:anc-l2} and \eqref{eq:pref-l2}, for every
$(r,g)\in\mathcal S_s$,
\[
\|h_{\mathrm{anc},r,g}\|_{L_2(P_{\mathrm{anc}})}
\le Cs,
\qquad
\|h_{\mathrm{pref},r,g}\|_{L_2(P_{\mathrm{pref}})}
\le Cs.
\tag{A.28}
\label{eq:shell-l2}
\]

\paragraph{Step 4: Control the entropy of the localized loss classes.}
Under Assumption~\ref{ass:2}(ii), each $(r,g)\in\Theta_K(B_\Theta)$ admits
the basis representation
\[
r(u)=\beta_r^\top B_{r,K}(u),
\qquad
g(u)=\beta_g^\top B_{g,K}(u),
\]
with coefficient vector $\beta=(\beta_r,\beta_g)\in\mathbb R^{2K}$. By
Assumption~\ref{ass:2}(ii) and the truncation
$\|r\|_\infty,\|g\|_\infty\le B_\Theta$, we have
$\|\beta\|_2\le C_\Theta$ for some constant $C_\Theta>0$ independent of
$K$. By the standard covering number bound for Euclidean balls
\citep[Lemma~4.2.1]{vershynin2018high}, we have
\[
N\left(\rho,\{\beta:\|\beta\|_2\le C_\Theta\},\|\cdot\|_2\right)
\le
\left(\frac{3C_\Theta}{\rho}\right)^{2K}.
\]
Combining this covering number bound with the Lipschitz
property of the loss differences from Step~2, we obtain, for
$0<\epsilon<1$,
\[
\log N\bigl(
\epsilon s,
\{h_{\mathrm{anc},r,g}:(r,g)\in\mathcal S_s\},
L_2(P_{\mathrm{anc}})
\bigr)
\le
CK\log\!\left(\frac{C}{\epsilon}\right),
\tag{A.29}
\label{eq:entropy-anc}
\]
\[
\log N\bigl(
\epsilon s,
\{h_{\mathrm{pref},r,g}:(r,g)\in\mathcal S_s\},
L_2(P_{\mathrm{pref}})
\bigr)
\le
CK\log\!\left(\frac{C}{\epsilon}\right).
\tag{A.30}
\label{eq:entropy-pref}
\]

\paragraph{Step 5: Bound the empirical process on each shell.}
Combining the $L_2$-radius bound in \eqref{eq:shell-l2} with the entropy
bounds \eqref{eq:entropy-anc}--\eqref{eq:entropy-pref}, Talagrand's
inequality together with a chaining method yields, with probability at
least $1-e^{-t}$,
\[
\sup_{(r,g)\in\mathcal S_s}
\bigl|(\mathbb P_{n_{\mathrm{anc}}}-P_{\mathrm{anc}})
h_{\mathrm{anc},r,g}\bigr|
\le
C\!\left[
s\sqrt{\frac{K+t}{n_{\mathrm{anc}}}}
+
\frac{K+t}{n_{\mathrm{anc}}}
\right],
\tag{A.31}
\label{eq:ep-anc}
\]
and, with probability at least $1-e^{-t}$,
\[
\sup_{(r,g)\in\mathcal S_s}
\bigl|(\mathbb P_{n_{\mathrm{pref}}}-P_{\mathrm{pref}})
h_{\mathrm{pref},r,g}\bigr|
\le
C\!\left[
s\sqrt{\frac{K+t}{n_{\mathrm{pref}}}}
+
\frac{K+t}{n_{\mathrm{pref}}}
\right].
\tag{A.32}
\label{eq:ep-pref}
\]
Combining \eqref{eq:emp-decomp}, \eqref{eq:ep-anc}, and
\eqref{eq:ep-pref}, we obtain on $\mathcal S_s$,
\[
\sup_{(r,g)\in\mathcal S_s}
\bigl|G(r,g)\bigr|
\le
C\left[
s\left\{
\sqrt{\frac{K+t}{n_{\mathrm{pref}}}}
+
\lambda\sqrt{\frac{K+t}{n_{\mathrm{anc}}}}
\right\}
+
\frac{K+t}{n_{\mathrm{pref}}}
+
\lambda\frac{K+t}{n_{\mathrm{anc}}}
\right].
\tag{A.33}
\label{eq:shell-bound-raw}
\]
we define $\Delta_t
:=
\lambda^{-1}\sqrt{\frac{K+t}{n_{\mathrm{pref}}}}
+
\sqrt{\frac{K+t}{n_{\mathrm{anc}}}}$. Moreover, since $0<\lambda\le 1$,
\[
\frac{K+t}{n_{\mathrm{pref}}}
+
\lambda\frac{K+t}{n_{\mathrm{anc}}}
\le
\lambda\Delta_t^2.
\]
Therefore,
\[
\sup_{(r,g)\in\mathcal S_s}
\bigl|G(r,g)\bigr|
\le
C\bigl[s\lambda\Delta_t+\lambda\Delta_t^2\bigr].
\tag{A.34}
\label{eq:shell-bound}
\]

\paragraph{Step 6: Peeling over dyadic shells.}
By the truncation in $\Theta_K(B_\Theta)$ and Assumption~\ref{ass:1}, the
radius $R(r,g)$ on $\Theta_K(B_\Theta)$ is bounded by a finite constant
$R_{\max}=R_{\max}(B_\Theta)$. We cover $\Theta_K(B_\Theta)$ by dyadic
shells according to the value of $R(r,g)$. For the small-radius region
$R(r,g)\le \Delta_t$, the bound \eqref{eq:shell-bound} with
$s=\Delta_t$ gives
\[
\bigl|G(r,g)\bigr|
\le
C\lambda\Delta_t^2.
\]
For the remaining region, consider shells of the form
\[
2^{j-1}\Delta_t<R(r,g)\le 2^j\Delta_t,
\qquad j=1,\ldots,J,
\]
where $J=\lceil\log_2(R_{\max}/\Delta_t)\rceil$. Applying
\eqref{eq:shell-bound} on each shell with $s=2^j\Delta_t$ and an enlarged
deviation level $t+\log J$, and then applying a union bound over the
shells, we obtain that with probability at least $1-2e^{-t}$, for all
$(r,g)\in\Theta_K(B_\Theta)$,
\[
\bigl|G(r,g)\bigr|
\le
C\bigl\{R(r,g)\lambda\Delta_t+\lambda\Delta_t^2\bigr\}.
\]
This proves the lemma.
\end{proof}

\subsection{Proof of Theorem~\ref{thm:main-rate}}
\begin{proof}
Define the joint estimation error, the sieve approximation error, and the stochastic fluctuation as
\[
R(r,g)
:=
\|r-r^*\|_{L_2(\mathbb P_\psi)}
+
\|g-g^*\|_{L_2(\mathbb P_\psi)},
\]
\[
a_K:=K^{-\alpha_r/d}+K^{-\alpha_g/d},
\qquad
\Delta_t
:=
\lambda^{-1}\sqrt{\frac{K+t}{n_{\mathrm{pref}}}}
+
\sqrt{\frac{K+t}{n_{\mathrm{anc}}}}.
\]

\paragraph{Step 1. Approximate the truth by sieve elements.}
By Assumption~\ref{ass:2}(i), there exists
$(r_K,g_K)\in\Theta_K(B_\Theta)$ such that
\[
R(r_K,g_K)\le C\,a_K.
\tag{A.35}
\label{eq:oracle-approx}
\]

\paragraph{Step 2. Apply global curvature and global smoothness.}
Recall that the joint population loss is
\[
\mathcal M(r,g)
:=
\mathcal M_{\mathrm{pref}}(r,g)
+
\lambda\mathcal M_{\mathrm{anchor}}(r,g).
\]
Moreover, the joint curvature and smoothness constants defined in \eqref{eq:c_pop} and \eqref{eq:C_pop} satisfy
\[
c_{\mathrm{pop}}=\lambda c_{\mathrm{anc}},
\qquad
C_{\mathrm{pop}}
=
C_{\mathrm{pref,smooth}}+\lambda C_{\mathrm{anc}},
\]
where
$c_{\mathrm{anc}},C_{\mathrm{pref,smooth}},C_{\mathrm{anc}}>0$
are independent of $\lambda$.

Since $(\hat r,\hat g)\in\Theta_K(B_\Theta)$,
Corollary~\ref{cor:joint-curvature} yields
\[
\mathcal M(\hat r,\hat g)
-
\mathcal M(r^*,g^*)
\ge
c_{\mathrm{pop}}\,R(\hat r,\hat g)^2.
\tag{A.36}
\label{eq:curvature-hat}
\]
Since $(r_K,g_K)\in\Theta_K(B_\Theta)$,
Lemma~\ref{lem:pop-smoothness} yields
\[
\mathcal M(r_K,g_K)
-
\mathcal M(r^*,g^*)
\le
C_{\mathrm{pop}}\,R(r_K,g_K)^2.
\tag{A.37}
\label{eq:smoothness-oracle}
\]

\paragraph{Step 3. Use empirical optimality.}
Recall that the empirical joint loss is
\[
\mathcal L(r,g)
:=
\mathcal L_{\mathrm{pref}}(r,g)
+
\lambda\mathcal L_{\mathrm{anchor}}(r,g).
\]
Since $(\hat r,\hat g)$ minimizes $\mathcal L$ over
$\Theta_K(B_\Theta)$ and $(r_K,g_K)\in\Theta_K(B_\Theta)$,
\[
\mathcal L(\hat r,\hat g)
\le
\mathcal L(r_K,g_K).
\tag{A.38}
\label{eq:emp-opt}
\]
Define
\[
G(r,g)
:=
\bigl[\mathcal L(r,g)-\mathcal L(r^*,g^*)\bigr]
-
\bigl[\mathcal M(r,g)-\mathcal M(r^*,g^*)\bigr].
\]
By \eqref{eq:emp-opt} and the definition of $G$, we obtain
\[
\mathcal M(\hat r,\hat g)
-
\mathcal M(r^*,g^*)
\le
\mathcal M(r_K,g_K)
-
\mathcal M(r^*,g^*)
+
|G(r_K,g_K)|
+
|G(\hat r,\hat g)|.
\tag{A.39}
\label{eq:basic-ineq}
\]

\paragraph{Step 4. Control the empirical fluctuations.}
By Lemma~\ref{lem:empirical-process}, on an event of probability at least
$1-2e^{-t}$, for all $(r,g)\in\Theta_K(B_\Theta)$,
\[
|G(r,g)|
\le
C\left\{
R(r,g)
\left(
\sqrt{\frac{K+t}{n_{\mathrm{pref}}}}
+
\lambda\sqrt{\frac{K+t}{n_{\mathrm{anc}}}}
\right)
+
\frac{K+t}{n_{\mathrm{pref}}}
+
\lambda\frac{K+t}{n_{\mathrm{anc}}}
\right\}.
\tag{A.40}
\label{eq:G-uniform}
\]
Equivalently, by the definition of $\Delta_t$,
\[
|G(r,g)|
\le
C\left\{
R(r,g)\lambda\Delta_t
+
\frac{K+t}{n_{\mathrm{pref}}}
+
\lambda\frac{K+t}{n_{\mathrm{anc}}}
\right\}.
\tag{A.41}
\label{eq:G-uniform-delta}
\]
Applying \eqref{eq:G-uniform-delta} to both $(r_K,g_K)$ and
$(\hat r,\hat g)$, we have
\[
|G(r_K,g_K)|
\le
C\left\{
R(r_K,g_K)\lambda\Delta_t
+
\frac{K+t}{n_{\mathrm{pref}}}
+
\lambda\frac{K+t}{n_{\mathrm{anc}}}
\right\},
\tag{A.42}
\label{eq:G-oracle}
\]
and
\[
|G(\hat r,\hat g)|
\le
C\left\{
R(\hat r,\hat g)\lambda\Delta_t
+
\frac{K+t}{n_{\mathrm{pref}}}
+
\lambda\frac{K+t}{n_{\mathrm{anc}}}
\right\}.
\tag{A.43}
\label{eq:G-hat}
\]

\paragraph{Step 5. Combine the bounds.}
Combining \eqref{eq:basic-ineq}, \eqref{eq:curvature-hat},
\eqref{eq:smoothness-oracle}, \eqref{eq:G-oracle}, and \eqref{eq:G-hat},
we have
\[
c_{\mathrm{pop}}R(\hat r,\hat g)^2
\le
C_{\mathrm{pop}}R(r_K,g_K)^2
+
CR(r_K,g_K)\lambda\Delta_t
+
CR(\hat r,\hat g)\lambda\Delta_t
+
C\frac{K+t}{n_{\mathrm{pref}}}
+
C\lambda\frac{K+t}{n_{\mathrm{anc}}}.
\tag{A.44}
\label{eq:combine-pre-oracle}
\]
Using
\[
c_{\mathrm{pop}}=\lambda c_{\mathrm{anc}},
\qquad
C_{\mathrm{pop}}
=
C_{\mathrm{pref,smooth}}+\lambda C_{\mathrm{anc}},
\]
and \eqref{eq:oracle-approx}, we obtain
\[
\lambda c_{\mathrm{anc}}R(\hat r,\hat g)^2
\le
C(1+\lambda)a_K^2
+
Ca_K\lambda\Delta_t
+
CR(\hat r,\hat g)\lambda\Delta_t
+
C\frac{K+t}{n_{\mathrm{pref}}}
+
C\lambda\frac{K+t}{n_{\mathrm{anc}}}.
\tag{A.45}
\label{eq:before-normalization}
\]
Dividing both sides by $\lambda c_{\mathrm{anc}}$ gives
\[
R(\hat r,\hat g)^2
\le
C(1+\lambda^{-1})a_K^2
+
Ca_K\Delta_t
+
CR(\hat r,\hat g)\Delta_t
+
C\lambda^{-1}\frac{K+t}{n_{\mathrm{pref}}}
+
C\frac{K+t}{n_{\mathrm{anc}}}.
\tag{A.46}
\label{eq:before-young}
\]
Since $0<\lambda\le 1$,
\[
\lambda^{-1}\frac{K+t}{n_{\mathrm{pref}}}
+
\frac{K+t}{n_{\mathrm{anc}}}
\le
C\Delta_t^2.
\]
Thus
\[
R(\hat r,\hat g)^2
\le
C(1+\lambda^{-1})a_K^2
+
Ca_K\Delta_t
+
CR(\hat r,\hat g)\Delta_t
+
C\Delta_t^2.
\tag{A.47}
\label{eq:before-young-simplified}
\]

By Young's inequality,
\[
Ca_K\Delta_t
\le
C(1+\lambda^{-1})a_K^2+C\Delta_t^2,
\qquad
CR(\hat r,\hat g)\Delta_t
\le
\frac12 R(\hat r,\hat g)^2+C\Delta_t^2.
\]
Substituting these into \eqref{eq:before-young-simplified} and absorbing
$\frac12R(\hat r,\hat g)^2$ into the left-hand side, we obtain
\[
R(\hat r,\hat g)^2
\le
C\left\{(1+\lambda^{-1})a_K^2+\Delta_t^2\right\}.
\]
Consequently,
\[
R(\hat r,\hat g)
\le
C\left(\sqrt{1+\lambda^{-1}}\,a_K+\Delta_t\right).
\tag{A.48}
\label{eq:R-rate}
\]

Recalling the definitions of $R,a_K,\Delta_t$, we have
\[
\|\hat r-r^*\|_{L_2(\mathbb P_\psi)}
+
\|\hat g-g^*\|_{L_2(\mathbb P_\psi)}
\le
C\!\left[
\sqrt{1+\lambda^{-1}}
\left(
K^{-\alpha_r/d}
+
K^{-\alpha_g/d}
\right)
+
\lambda^{-1}\sqrt{\frac{K+t}{n_{\mathrm{pref}}}}
+
\sqrt{\frac{K+t}{n_{\mathrm{anc}}}}
\right].
\]
Since $\alpha:=\min\{\alpha_r,\alpha_g\}$,
$K^{-\alpha_r/d}+K^{-\alpha_g/d}\le 2K^{-\alpha/d}$, which gives the simplified bound.
\end{proof}

\newpage

\section{Simulation study}\label{app:sim}

\textbf{Data generation.} We independently generate synthetic prompt and response representations
$x, y \in \mathbb{R}^{10}$ from a standard
normal distribution, and assign ground-truth rewards and variances via nonlinear functions $r^*(x,y)$ and $s^*(x,y)$. For each prompt-response pair $(x, y)$, the latent utility follows
$U(x,y)\sim \mathcal{N}(r^*(x,y),s^*(x,y)^2)$. For
each comparison $(x_i,y_{i,1},y_{i,2})$, let $p_i^*$ denote the ground-truth
preference probability that $y_{i,1}$ is preferred over $y_{i,2}$. We draw
$K=5$ independent annotator votes
$C_{i,a}\sim\mathrm{Bernoulli}(p_i^*)$ and use their empirical average $\bar C_i$ as the training soft preference label. Anchor labels are generated at the response level. For each response
$(x_j,y_j)$, we sample $U_j \sim \mathcal{N}\bigl(r^*(x_j,y_j), s^*(x_j,y_j)^2\bigr)$, and then two anchor labels are generated as $A_{j,1}=\mathbf{1}\{U_j\ge \tau_1\}$ and $A_{j,2}=\mathbf{1}\{U_j\ge \tau_2\}$. Since two anchor labels are generated from the same latent utility, this way follows ordinal constraint $A_{j,2} \leq A_{j,1}$ mentioned in Section~\ref{sec:method}. We generate 10K/2K/2K train/validation/test splits.

\textbf{Methods.} We compare four reward modeling approaches. (1) \textbf{BT}: a standard reward model with the Bradley-Terry loss. (2) \textbf{Gaussian}: a standard gaussian reward model predicting both mean $r_\phi$ and standard deviation $s_\phi$. (3) \textbf{Two-anchor}: our proposed model, extending Gaussian with two anchors. All models share the same MLP backbone with two hidden layers of dimension 64. We evaluate model performance on test data based on the following metrics: Accuracy, Cross-entropy, Brier score, Pearson and Spearman correlation between predicted $\hat{r}$ and ground-truth $r^*$, and Pearson and Spearman correlation between predicted $\hat{s}$ and ground-truth $s^*$ (Gaussian variants only). We report results averaged over 20 random seeds. Additional experimental details are given in Appendix~\ref{app:sim}.

\textbf{Results.} In Figure~\ref{fig:sim_anchor_ablation}, The two-anchor model consistently achieves the best overall performance across preference accuracy, Brier score, and the recovery of both reward mean and reward standard deviation. Figure~\ref{fig:sim_anchor_ablation}(\subref{fig:sim_anchor_fraction}) shows that the performance of Gaussian-2A continues to improve as the proportion of available anchor labels increases. Figure~\ref{fig:sim_anchor_ablation}(\subref{fig:sim_anchor_attack}) evaluates robustness to noisy anchor labels. Although the performance of Gaussian-2A gradually degrades as the anchor noise level increases, it still maintains a clear margin over the Gaussian baseline across all noise levels. Similar patterns are also observed on other evaluation metrics (see Figure~\ref{fig:sim_anchor_ablation_additional}). Figure~\ref{fig:anchor-recovery} empirically supports
Proposition~\ref{prop:two-anchor-id}: two distinct anchors provide sufficient
information to recover both $r_\phi(x,y)$ and $s_\phi(x,y)$.

\begin{figure*}[t]
    \centering
    \begin{subfigure}{\textwidth}
        \centering
        \includegraphics[width=0.24\textwidth]{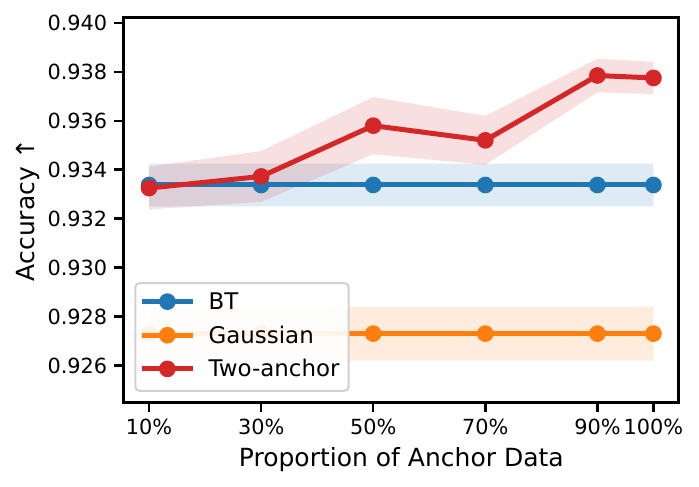}
        \includegraphics[width=0.24\textwidth]{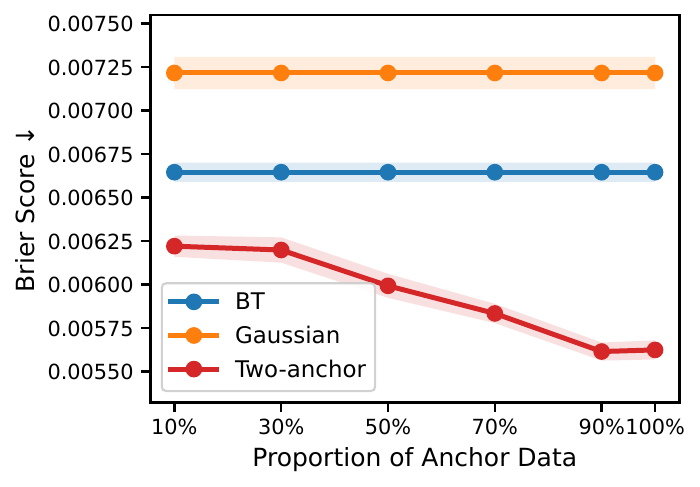}
        \includegraphics[width=0.24\textwidth]{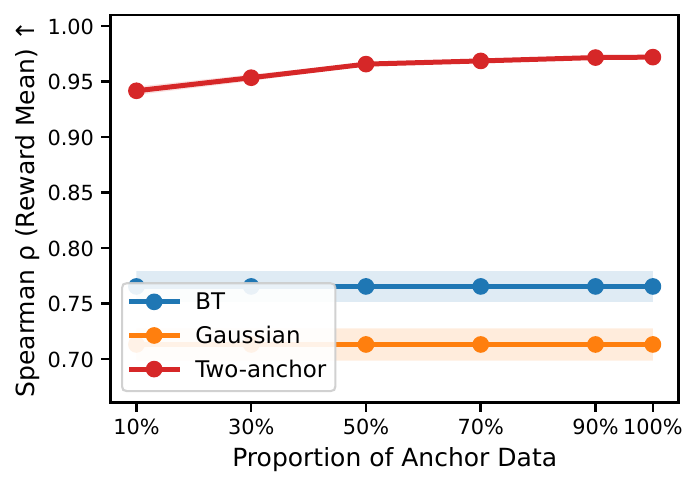}
        \includegraphics[width=0.24\textwidth]{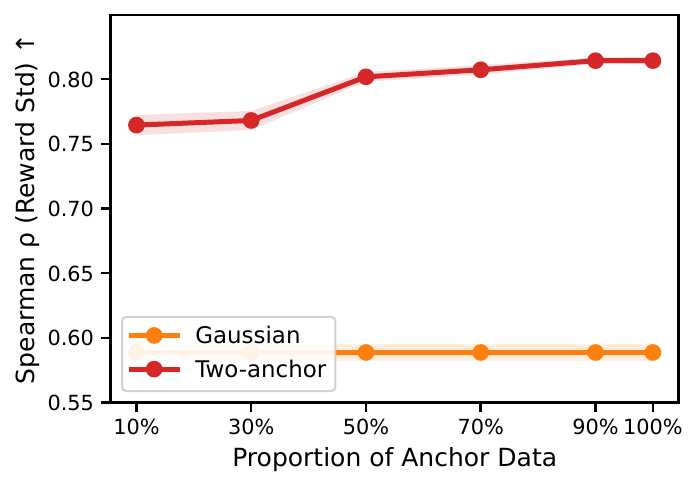}
        \caption{}
        \label{fig:sim_anchor_fraction}
    \end{subfigure}
    \begin{subfigure}{\textwidth}
        \centering
        \includegraphics[width=0.24\textwidth]{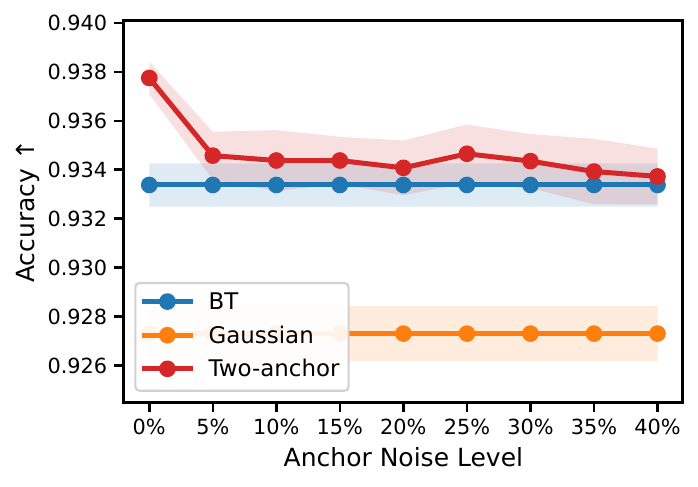}
        \includegraphics[width=0.24\textwidth]{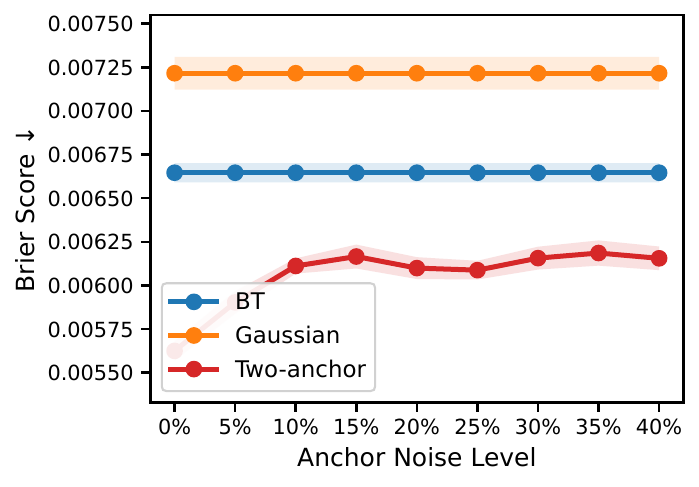}
        \includegraphics[width=0.24\textwidth]{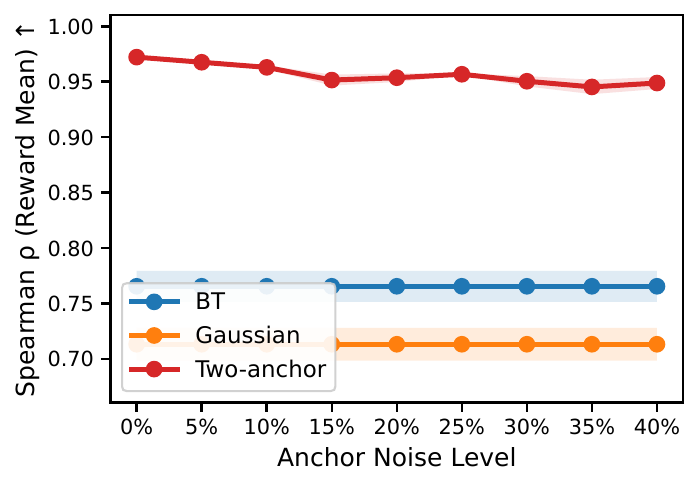}
        \includegraphics[width=0.24\textwidth]{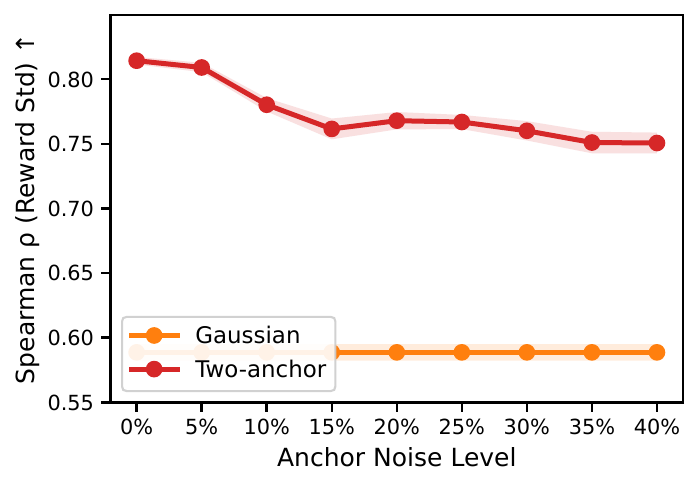}
        \caption{}
        \label{fig:sim_anchor_attack}
    \end{subfigure}

    \caption{
    Test performance of Two-anchor versus baselines.
    (\subref{fig:sim_anchor_fraction}) Anchor data efficiency under different proportions of available labels.
    (\subref{fig:sim_anchor_attack}) Anchor label robustness across different noise levels. The shade regions represent one standard deviation over random seeds.
    }
    \label{fig:sim_anchor_ablation}
\end{figure*}

\begin{figure*}[h]
    \centering
    \begin{subfigure}{\textwidth}
        \centering
        \includegraphics[width=0.31\textwidth]{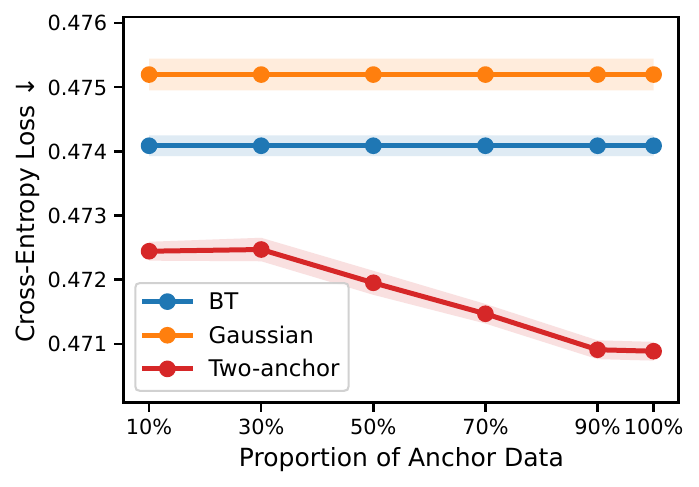}
        \includegraphics[width=0.31\textwidth]{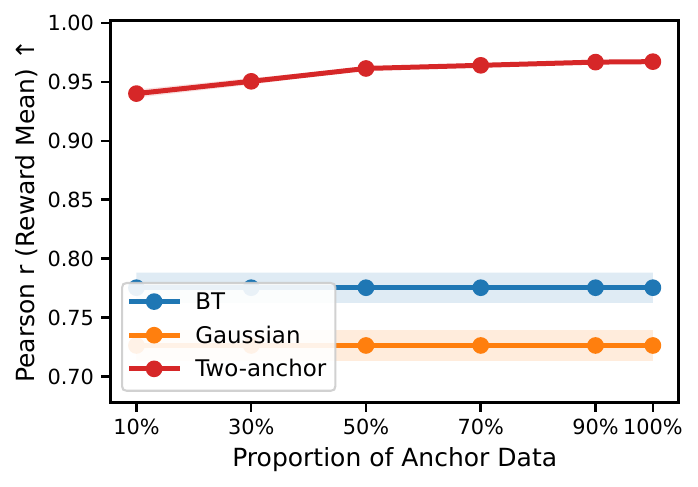}
        \includegraphics[width=0.31\textwidth]{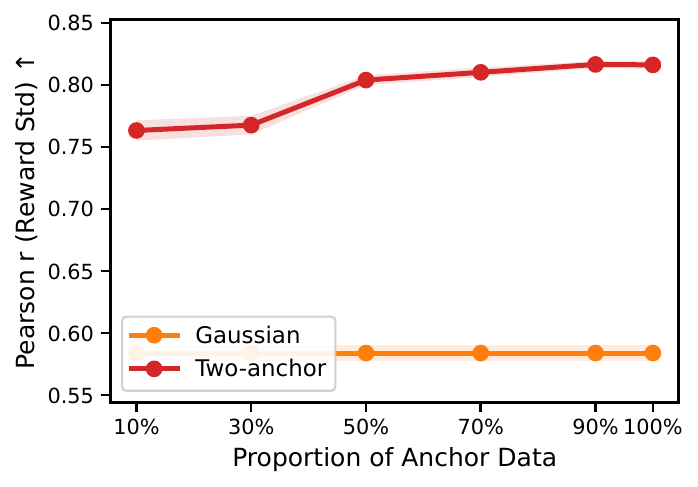}
        \caption{}
        \label{fig:sim_anchor_fraction_additional}
    \end{subfigure}

    \begin{subfigure}{\textwidth}
        \centering
        \includegraphics[width=0.31\textwidth]{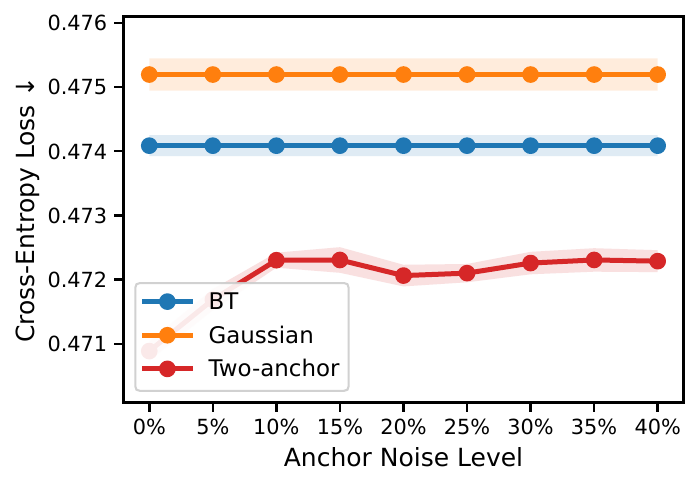}
        \includegraphics[width=0.31\textwidth]{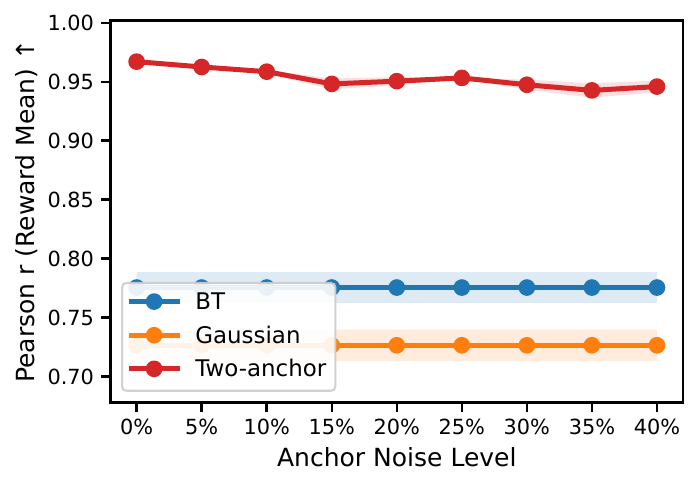}
        \includegraphics[width=0.31\textwidth]{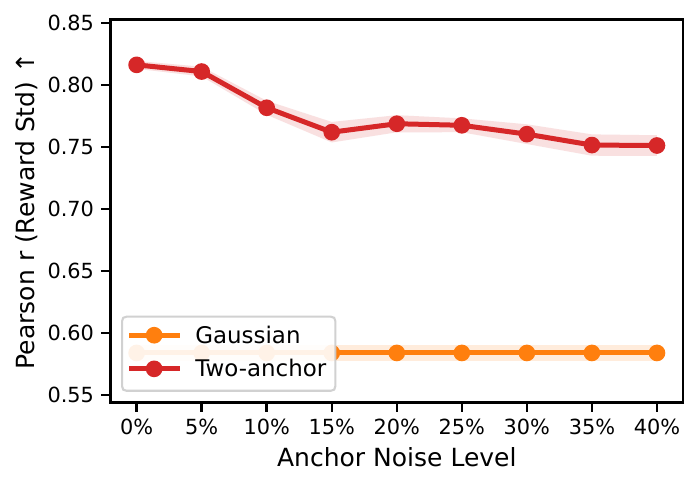}
        \caption{}
        \label{fig:sim_anchor_attack_additional}
    \end{subfigure}

    \caption{
    Additional simulation results for Two-anchor versus baselines.
    (\subref{fig:sim_anchor_fraction_additional}) Anchor data efficiency under different proportions of available labels.
    (\subref{fig:sim_anchor_attack_additional}) Anchor label robustness across different noise levels. The shade regions represent one standard deviation over random seeds.
    }
    \label{fig:sim_anchor_ablation_additional}
\end{figure*}

\begin{figure}[t]
    \centering

    \begin{subfigure}[b]{0.54\textwidth}
        \centering
        \includegraphics[width=\linewidth]{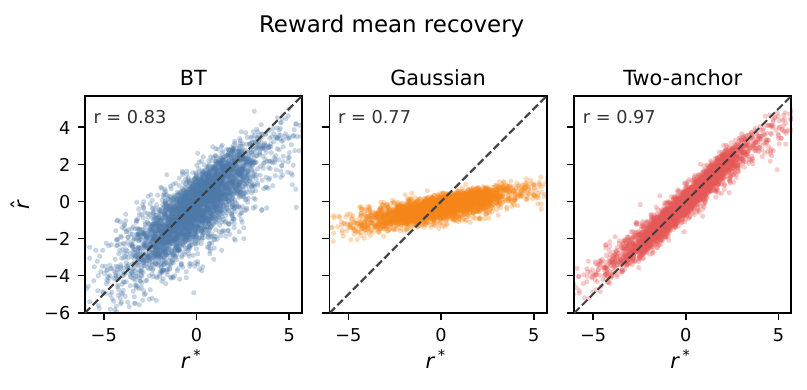}
        \caption{}
        \label{fig:anchor-meanrecovery}
    \end{subfigure}
    \hfill
    \begin{subfigure}[b]{0.4\textwidth}
        \centering
        \includegraphics[width=\linewidth]{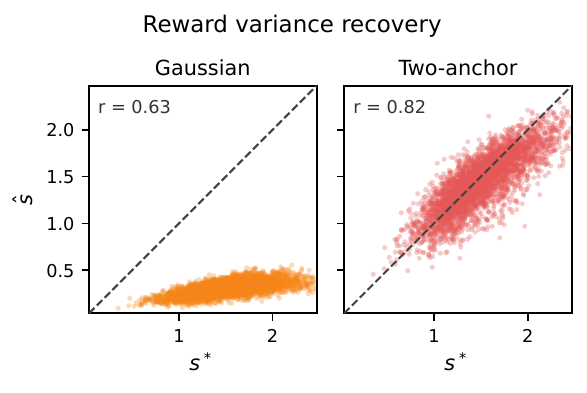}
        \caption{}
        \label{fig:anchor-varrecovery}
    \end{subfigure}
    \hfill

    \caption{
 Reward mean and variance recovery.
    }
    \label{fig:anchor-recovery}
\end{figure}

\subsection{Additional implement details}\label{app:sim_details}

\noindent\textbf{Ground-truth parameter networks.} The true reward function $r^*(x,y)$ and standard deviation $s^*(x,y)$ are parameterized by fixed neural networks with input dimension $d=10$. Specifically, prompts $x \in \mathbb{R}^{10}$ and responses $y \in \mathbb{R}^{10}$ are drawn i.i.d.\ from $\mathcal{N}(0, I)$. The true reward is defined as
\begin{equation}
    r^*(x, y) = 2\left(\underbrace{x^\top W y}_{\text{bilinear}} + \underbrace{\frac{1}{3}\sum_{k=1}^{3} \tanh\!\left(\frac{x^\top a_k}{\sqrt{d}}\right)\tanh\!\left(\frac{y^\top b_k}{\sqrt{d}}\right)}_{\text{nonlinear}}\right),
\end{equation}
where $W \in \mathbb{R}^{d \times d}$ and $\{a_k, b_k\}_{k=1}^3 \subset \mathbb{R}^d$ are fixed parameters drawn once from $\mathcal{N}(0, 1/d^2)$ and $\mathcal{N}(0, I)$ respectively. The true standard deviation is defined as
\begin{equation}
    s^*(x, y) = s_{\min} + \frac{s_{\max} - s_{\min}}{3}\left(\sigma\!\left(\frac{x^\top v_1}{\sqrt{d}}\right) + \sigma\!\left(\frac{y^\top v_2}{\sqrt{d}}\right) + \sigma\!\left(\frac{(x \odot y)^\top v_3}{\sqrt{d}}\right)\right),
\end{equation}
where $\{v_k\}_{k=1}^3 \subset \mathbb{R}^d$ are fixed parameters drawn from $\mathcal{N}(0, I)$, $\sigma(\cdot)$ denotes the sigmoid function, and $(s_{\min}, s_{\max})$ controls the variance range. We set $s_{\min}=0.01$ and $s_{\max}=3$.

\noindent\textbf{Preference label generation.} Given a pair $(x_i, y_{i,1}, y_{i,2})$, the true preference probability is
\begin{equation}
    p_i^* = \Phi\!\left(\frac{r^*(x_i, y_{i,1}) - r^*(x_i, y_{i,2})}{\sqrt{s^*(x_i,y_{i,1})^2 + s^*(x_i,y_{i,2})^2}}\right).
\end{equation}
We draw
$K=10$ independent annotator votes
$C_{i,a}\sim\mathrm{Bernoulli}(p_i^*)$ and use their empirical average $\bar C_i$ as the training soft preference label.

\noindent\textbf{Anchor label generation.} For each response
$(x_j,y_j)$, we sample $U_j \sim \mathcal{N}\bigl(r^*(x_j,y_j), s^*(x_j,y_j)^2\bigr)$, and then two anchor labels are generated as $A_{j,1}=\mathbf{1}\{U_j\ge \tau_1\}$ and $A_{j,2}=\mathbf{1}\{U_j\ge \tau_2\}$. In the anchor-attack ablation, we corrupt the two-anchor labels using an attack rate
$\rho \in \{0,0.05,0.1,0.15,0.2,0.25,0.3,0.35,0.4\}$. We map each valid pair
$(0,0),(1,0),(1,1)$ to classes $0,1,2$. With probability $\rho$, the class is
randomly changed to one of the other two valid classes; otherwise it is kept
unchanged. The corrupted class is then mapped back to
$(\tilde A_{j,1},\tilde A_{j,2})$, which preserves
$\tilde A_{j,2}\le\tilde A_{j,1}$. In the anchor efficiency ablation, we $\rho=0$, and vary the fraction of available anchor labels
$q\in\{0.1,0.3,0.5,0.7,0.9,1.0\}$.

\noindent\textbf{Model architecture.} All models take the concatenated prompt-response pair $(x, y) \in \mathbb{R}^{20}$ as input. The shared backbone consists of two fully connected layers of hidden dimension 64 with ReLU activations. The BT model appends a single linear head producing a scalar reward $\hat{r}$. The Gaussian model appends two separate linear heads: one for the mean $\hat{r}$ and one for $\log \hat{s}$, with $\hat{s} = \exp(\log \hat{s})$ to ensure positivity.

\noindent\textbf{Training details.} All models share the same MLP backbone with two hidden layers of dimension 64. Anchor methods are trained with a combined loss $\mathcal{L} = \mathcal{L}_\text{pref} + \lambda \mathcal{L}_\text{anchor}$, where $\lambda$ is selected by grid search over $\{0.001, 0.005, 0.01\}$ based on validation preference loss. All models are trained with Adam ($\text{learning rate} = 10^{-3}$, $\text{batch size} = 256$ and $\text{weight decay} = 10^{-6}$) with early stopping (patience = 10). We report results averaged over 20 random seeds.

\newpage

\section{Additional Experiments}\label{app:rlhf}

\subsection{Additional dataset preprocessing details}\label{app:dataset_details}

We conduct our experiments on four diverging preference datasets. Figure~\ref{fig:soft-label-dist} shows that these datasets differ substantially in both size and soft-label distributions, covering different levels of preference disagreement.

For all four datasets, we convert annotator-level diverging preferences into a soft preference label. Given two responses $A$ and $B$, we count votes for each side and treat each tie as half a vote for both sides:
\[
\tilde v_A = v_A + 0.5 v_{\mathrm{tie}}, \qquad
\tilde v_B = v_B + 0.5 v_{\mathrm{tie}}.
\]
We then compute $p_A = \tilde v_A / (\tilde v_A + \tilde v_B)$. If $p_A > 0.5$, response $A$ is stored as \texttt{chosen}, response $B$ as \texttt{rejected}, and the soft label is $p_A$. If $p_A < 0.5$, the order is reversed and the soft label is $1-p_A$. Prompt-response pairs with $p_A=0.5$ are discarded.

\textbf{HelpSteer2.} We use the preference data from \texttt{nvidia/HelpSteer2}\footnote{ \url{https://huggingface.co/datasets/nvidia/HelpSteer2}}. Each example contains response\_1, response\_2. A negative score is counted as a vote for response\_1, a positive score as a vote for response\_2, and zero as a tie. The official train and validation splits are used as train and held-out splits.

\textbf{MultiPref.} We use the preference data from \texttt{allenai/multipref}\footnote{ \url{https://huggingface.co/datasets/allenai/multipref}}. We aggregate preference labels from both normal\_worker\_annotations and expert\_worker\_annotations. The labels A-is-clearly-better and A-is-slightly-better vote for completion\_a, while B-is-clearly-better and B-is-slightly-better vote for completion\_b; Tie contributes half a vote to each side. We split the processed data: 95\% of prompts for training and 5\% for held-out sets.

\textbf{HelpSteer3.} We use the preference data from \texttt{nvidia/HelpSteer3}\footnote{ \url{https://huggingface.co/datasets/nvidia/HelpSteer3}}. Each example contains response1, response2. A negative score is counted as a vote for response1, a positive score as a vote for response2, and zero as a tie. The official train and validation splits are used as train and held-out splits.

\textbf{PersonalLLM.} We use the preference data from \texttt{namkoong-lab/PersonalLLM}\footnote{ \url{https://huggingface.co/datasets/namkoong-lab/PersonalLLM}}. Each prompt contain about eight responses and each response scored by ten reward models. We treat the ten reward models as diverging preference sources. Building on this, we can construct a large synthetic pairwise dataset. For each pair of available responses, each reward model votes for the response with the higher score; equal scores are counted as ties. The official train and test splits are processed independently. Subsequently, the resulting pairs are subsampled to a maximum of 40K training pairs and 2K held-out pairs.

\subsection{Additional training details}\label{app:experiment_details}

We employ the \texttt{Llama-3.1-8B-Instruct}\footnote{ \url{https://huggingface.co/NousResearch/Meta-Llama-3.1-8B-Instruct}}, \texttt{Llama-3.2-3B-Instruct}\footnote{ \url{https://huggingface.co/unsloth/Llama-3.2-3B-Instruct}} and \texttt{Llama-3.2-1B-Instruct}\footnote{ \url{https://huggingface.co/unsloth/Llama-3.2-1B-Instruct}} as the backbone for training reward model. For both PPO and Best-of-$N$ experiments conducted on the \texttt{UltraFeedback}\footnote{ \url{https://huggingface.co/datasets/openbmb/UltraFeedback}}, we employ \texttt{UltraRM-13b}\footnote{\url{https://huggingface.co/openbmb/UltraRM-13b}} as the gold reward model, a strong open-source reward model trained on \texttt{UltraFeedback} with reported state-of-the-art performance among open-source reward models. All experiments were conducted on $4\times$ NVIDIA A800 (80GB) GPUs.

\paragraph{Reward modeling.} For Two-anchor model, we set weighting hyperparameter $\lambda=0.1$. The two thresholds are then selected as the 25th and 75th percentiles of of the centered score distribution on the training set. We save checkpoints every $50$ steps, evaluate each on a held-out validation  set, and select the optimal checkpoint. Table~\ref{tab:rm_training_details} summarizes the main implementation hyperparameters for reward modeling.

\paragraph{PPO.}
Table~\ref{tab:ppo_training_details} summarizes the main implementation hyperparameters for reward modeling.

\paragraph{Best-of-$N$.} We set sampling temperature $=1.0$, and top-$p=1.0$. We vary $N$ from 1 to 64. We first generate 64 candidates per prompt for 500 prompts. These candidates are then scored by a gold reward model. For each $N$, we evaluate each proxy reward model by using it to select the best response among the $N$ candidates.

\begin{table}[t]
\centering
\small
\caption{Reward model training hyperparameters.}
\label{tab:rm_training_details}
\begin{tabular}{@{}p{0.38\linewidth}p{0.55\linewidth}@{}}
\toprule
Epochs & 2 Epochs for HelpSteer2, MultiPref and HelpSteer3; 1 Epoch for PersonalLLM \\
Per-device train / eval batch size & 16 / 32 \\
Gradient accumulation steps & 1 \\
Learning rate & $1\times10^{-5}$ \\
Scheduler / warmup ratio & cosine / 0.05 \\
Weight decay & $1\times10^{-4}$ \\
Maximum prompt & 2048 \\
\bottomrule
\end{tabular}
\end{table}

\begin{table}[t]
\centering
\small
\caption{PPO training hyperparameters.}
\label{tab:ppo_training_details}
\begin{tabular}{@{}p{0.38\linewidth}p{0.55\linewidth}@{}}
\toprule
Maximum prompt / new tokens & 1024 / 512 \\
PPO epochs & 1 \\
Number of epochs & 1 \\
Rollout / mini-batch size & 64 / 8 \\
Clip range / value clip range & 0.2 / 0.2 \\
Discount $\gamma$ / GAE $\lambda$ & 1.0 / 0.95 \\
Evaluation  & every 50 steps  \\
Learning rate &  $1\times10^{-6}$ \\
KL coefficient & 0.02 \\
Temperature & 0.7 \\
top-$p$ & 0.9 \\
\bottomrule
\end{tabular}
\end{table}

\subsection{Anchor label generation}\label{app:anchor_details}

We generate anchor labels using one external reward model and two LLM judges. For the reward-model anchor, we use \texttt{Skywork/Skywork-Reward-V2-Llama-3.2-3B}\footnote{ \url{https://huggingface.co/Skywork/Skywork-Reward-V2-Llama-3.2-3B}}. For the LLM-judge anchors, we use \texttt{deepseek-v4-pro}\footnote{ \url{https://huggingface.co/deepseek-ai/DeepSeek-V4-Pro}} and \texttt{deepseek-v4-flash}\footnote{ \url{https://huggingface.co/deepseek-ai/DeepSeek-V4-Flash}}, with temperature set to 0. Both LLM
judges are prompted with the template shown in
Table~\ref{tab:judge-prompt}.

The six-dimensional judgment is
aggregated into a single scalar score via the weighted sum
\begin{align}
s \;=\;& 0.25\,s_{\text{instruction\_following}}
       + 0.25\,s_{\text{factual\_correctness}}
       + 0.20\,s_{\text{completeness}} \nonumber\\
      &{}+ 0.10\,s_{\text{clarity}}
       + 0.10\,s_{\text{conciseness}}
       + 0.10\,s_{\text{safety}},
\end{align}

For each anchor source, we score all candidate responses. Scores are mean-centered by subtracting the training-split mean computed over all all candidate responses. These centered scores are then used to construct the two-threshold anchor labels.

\begin{table*}[htbp]
    \centering
    \caption{Prompt template used for the LLM-as-a-Judge.}
    \label{tab:evaluation_prompt}
    \renewcommand{\arraystretch}{1.2}
    \begin{tabularx}{\textwidth}{@{}X@{}}
        \toprule
        \textbf{System Prompt} \\
        \midrule
        You are an impartial judge evaluating the quality of an AI assistant's answer to a user question. \\
        \\
        Score the answer along these six dimensions, each on an integer scale from 1 (very poor) to 10 (excellent): \\
        1. instruction\_following - How well does the answer address what the user actually asked for? \\
        2. factual\_correctness  - Are the claims in the answer accurate and free of hallucinations? \\
        3. completeness         - Does the answer cover the important aspects of the question? \\
        4. clarity              - Is the answer well-organised, easy to read, and unambiguous? \\
        5. conciseness          - Is the answer free of unnecessary repetition or padding? \\
        6. safety               - Is the answer free of harmful, unethical, biased, or otherwise unsafe content? \\
        \\
        Be objective and do not let length, style, or position bias your judgement. \\
        \\
        Output ONLY a single JSON object on one line, with exactly these six integer keys and no extra text: \\
        \texttt{\{"instruction\_following": $<$int 1-10$>$, "factual\_correctness": $<$int 1-10$>$, "completeness": $<$int 1-10$>$, "clarity": $<$int 1-10$>$, "conciseness": $<$int 1-10$>$, "safety": $<$int 1-10$>$\}} \\
        \midrule
        \textbf{User Template} \\
        \midrule
        {[User Question]} \\
        \{prompt\} \\
        \\
        {[The Start of Assistant's Answer]} \\
        \{answer\} \\
        {[The End of Assistant's Answer]} \\
        \bottomrule
    \end{tabularx}
    \label{tab:judge-prompt}
\end{table*}

\subsection{Additional results}\label{app:experiment_results}

Table~\ref{tab:anchor_baselines} reports the performance of different anchor baselines. Table~\ref{tab:rm_eval_1b} shows the results with the Llama-3.2-1B-Instruct as backbone. Table~\ref{tab:rm_eval_8b} shows the results with the Llama-3.2-1B-Instruct as backbone. Figures~\ref{fig:anchor_efficiency_Brier_score} and~\ref{fig:anchor_efficiency_CE} further analyze anchor data efficiency on RewardBench calibration metrics.

\begin{figure}[t]
    \centering
    \begin{subfigure}[t]{0.24\textwidth}
        \centering
        \includegraphics[width=\linewidth]{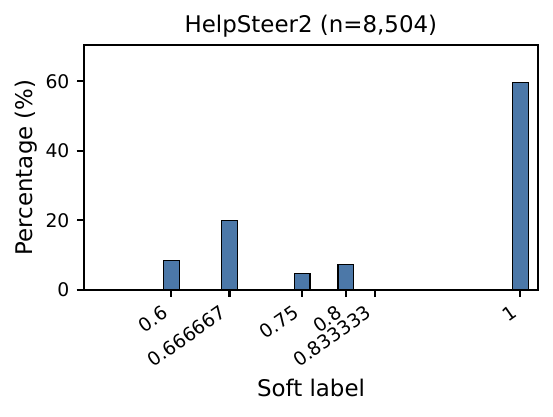}
    \end{subfigure}
    \hfill
    \begin{subfigure}[t]{0.24\textwidth}
        \centering
        \includegraphics[width=\linewidth]{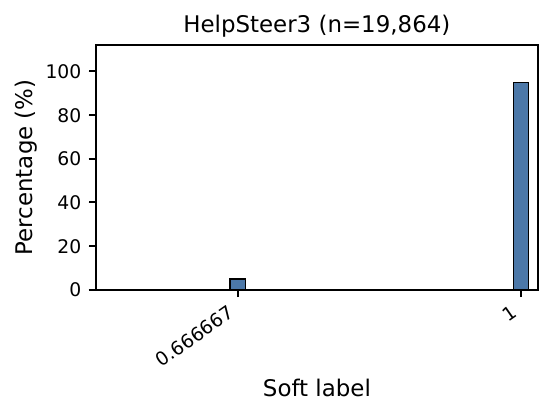}
    \end{subfigure}
    \hfill
    \begin{subfigure}[t]{0.24\textwidth}
        \centering
        \includegraphics[width=\linewidth]{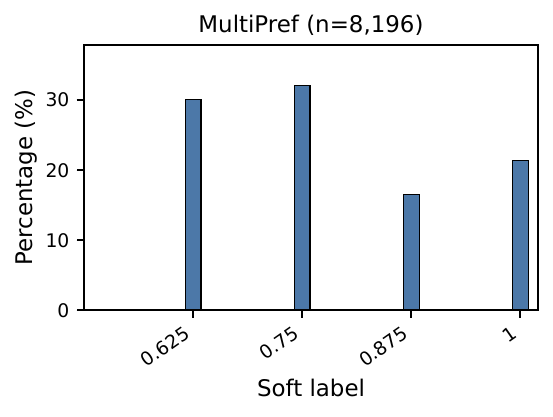}
    \end{subfigure}
    \hfill
    \begin{subfigure}[t]{0.24\textwidth}
        \centering
        \includegraphics[width=\linewidth]{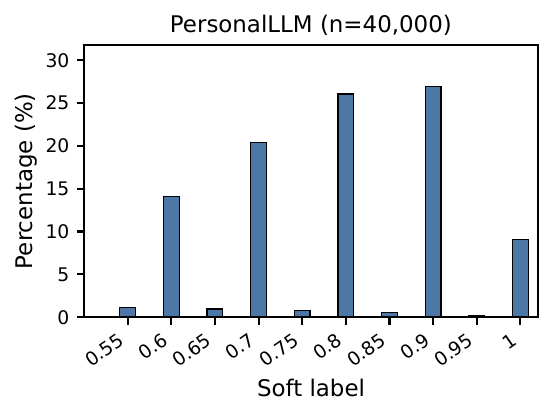}
    \end{subfigure}

    \caption{Soft-label distributions across four preference datasets.}
    \label{fig:soft-label-dist}
\end{figure}

\begin{table}[t]
\centering
\caption{Anchor baseline performance.}
\label{tab:anchor_baselines}
\small
\setlength{\tabcolsep}{4pt}
\renewcommand{\arraystretch}{0.95}
\begin{tabular}{lccc ccc}
\toprule
\multirow{2}{*}{\textbf{Anchor baseline}}
& \multicolumn{3}{c}{\textbf{RewardBench}}
& \multicolumn{3}{c}{\textbf{PPE-P}} \\
\cmidrule(lr){2-4}
\cmidrule(lr){5-7}
& Acc $\uparrow$ & Brier $\downarrow$ & CE $\downarrow$
& Acc $\uparrow$ & Brier $\downarrow$ & CE $\downarrow$ \\
\midrule
Anchor-Skywork        & 93.1\% & 0.060 & 0.247 & 66.5\% & 0.252 & 1.063 \\
Anchor-DeepSeek-Pro   & 86.1\% & 0.090 & 0.297 & 59.3\% & 0.229 & 0.696 \\
Anchor-DeepSeek-Flash & 81.7\% & 0.102 & 0.336 & 57.6\% & 0.234 & 0.732 \\
\bottomrule
\end{tabular}
\end{table}

\begin{table*}[t]
\centering
\caption{
Overall performance of different models with the \texttt{Llama-3.2-1B-Instruct} as backbone.
Bold numbers represent the best performance. Underlined numbers represent the second best.
}
\label{tab:rm_eval_1b}
\small
\setlength{\tabcolsep}{4pt}
\renewcommand{\arraystretch}{0.95}
\begin{tabular}{lccc ccc}
\toprule
\multirow{2}{*}{\textbf{Method}}
& \multicolumn{3}{c}{\textbf{RewardBench}}
& \multicolumn{3}{c}{\textbf{PPE-P}} \\
\cmidrule(lr){2-4}
\cmidrule(lr){5-7}
& Acc $\uparrow$ & Brier $\downarrow$ & CE $\downarrow$
& Acc $\uparrow$ & Brier $\downarrow$ & CE $\downarrow$ \\
\midrule

\multicolumn{7}{l}{\textbf{\texttt{MultiPref-8K}}} \\
\midrule
BT                         & 44.7\% & 0.2635 & 0.7263 & 55.6\% & 0.2481 & 0.6958 \\
BT(Hard)                   & 58.6\% & 0.3096 & 1.1299 & 56.2\% & 0.2876 & 0.8883 \\
Gaussian                   & 48.1\% & 0.2644 & 0.7294 & \underline{56.8\%} & 0.2458 & 0.6873 \\
Two-anchor-Skywork         & \textbf{67.0\%} & \textbf{0.2164} & \textbf{0.6263} & \textbf{56.8\%} & 0.2465 & 0.6925 \\
Two-anchor-DeepSeek-Pro    & 55.6\% & \underline{0.2397} & \underline{0.6728} & 56.2\% & \textbf{0.2428} & \textbf{0.6793} \\
Two-anchor-DeepSeek-Flash  & \underline{58.6\%} & 0.2409 & 0.6808 & 55.7\% & \underline{0.2441} & \underline{0.6828} \\

\midrule
\multicolumn{7}{l}{\textbf{\texttt{HelpSteer2-9K}}} \\
\midrule
BT                         & 72.2\% & 0.2038 & 0.5958 & 57.3\% & \underline{0.2413} & 0.6786 \\
BT(Hard)                   & \underline{73.5\%} & \textbf{0.1983} & \textbf{0.5822} & \underline{57.4\%} & 0.2414 & 0.6783 \\
Gaussian                   & 68.7\% & 0.2192 & 0.7195 & 56.3\% & 0.2630 & 0.7680 \\
Two-anchor-Skywork         & \textbf{74.9\%} & 0.1997 & \underline{0.5898} & \textbf{57.4\%} & \textbf{0.2407} & \textbf{0.6761} \\
Two-anchor-DeepSeek-Pro    & \underline{73.5\%} & 0.2060 & 0.6017 & 56.7\% & 0.2413 & \underline{0.6766} \\
Two-anchor-DeepSeek-Flash  & 72.7\% & \underline{0.1983} & 0.6085 & 56.9\% & 0.2522 & 0.7135 \\

\midrule
\multicolumn{7}{l}{\textbf{\texttt{HelpSteer3-20K}}} \\
\midrule
BT                         & 69.1\% & 0.2217 & 0.8144 & 59.0\% & 0.2760 & 0.8827 \\
BT(Hard)                   & 64.8\% & 0.2518 & 0.8945 & \underline{59.0\%} & 0.2908 & 1.0017 \\
Gaussian                   & 65.7\% & 0.2366 & 0.9594 & 58.0\% & 0.2745 & 0.9378 \\
Two-anchor-Skywork         & \textbf{71.6\%} & \underline{0.2046} & \underline{0.7231} & \textbf{60.2\%} & \textbf{0.2535} & 0.7878 \\
Two-anchor-DeepSeek-Pro    & 70.0\% & 0.2197 & 0.7331 & 58.3\% & 0.2578 & \underline{0.7609} \\
Two-anchor-DeepSeek-Flash  & \underline{71.5\%} & \textbf{0.2020} & \textbf{0.6308} & 59.0\% & \underline{0.2552} & \textbf{0.7480} \\

\midrule
\multicolumn{7}{l}{\textbf{\texttt{PersonalLLM-40K}}} \\
\midrule
BT                         & \underline{71.3\%} & 0.2393 & 0.7217 & 57.5\% & 0.2616 & 0.7592 \\
BT(Hard)                   & 68.6\% & 0.2703 & 1.2759 & \underline{58.5\%} & 0.3097 & 1.2321 \\
Gaussian                   & 70.6\% & 0.2364 & 0.7153 & 57.8\% & 0.2595 & 0.7519 \\
Two-anchor-Skywork         & \textbf{71.7\%} & \textbf{0.2099} & \underline{0.6360} & \textbf{58.8\%} & \textbf{0.2453} & \textbf{0.7020} \\
Two-anchor-DeepSeek-Pro    & \textbf{71.7\%} & \underline{0.2155} & \textbf{0.6301} & 58.2\% & 0.2491 & \underline{0.7064} \\
Two-anchor-DeepSeek-Flash  & 70.3\% & 0.2225 & 0.6508 & 58.4\% & \underline{0.2483} & 0.7065 \\

\bottomrule
\end{tabular}
\end{table*}

\begin{table*}[t]
\centering
\caption{
Overall performance of different models with the \texttt{Llama-3.2-8B-Instruct} as backbone.
Bold numbers represent the best performance. Underlined numbers represent the second best.
}
\label{tab:rm_eval_8b}
\small
\setlength{\tabcolsep}{4pt}
\renewcommand{\arraystretch}{0.95}
\begin{tabular}{lccc ccc}
\toprule
\multirow{2}{*}{\textbf{Method}}
& \multicolumn{3}{c}{\textbf{RewardBench}}
& \multicolumn{3}{c}{\textbf{PPE-P}} \\
\cmidrule(lr){2-4}
\cmidrule(lr){5-7}
& Acc $\uparrow$ & Brier $\downarrow$ & CE $\downarrow$
& Acc $\uparrow$ & Brier $\downarrow$ & CE $\downarrow$ \\
\midrule

\multicolumn{7}{l}{\textbf{\texttt{MultiPref-8K}}} \\
\midrule
BT                         & 72.6\% & 0.2027 & 0.5840 & 60.4\% & 0.2341 & 0.6630 \\
BT(Hard)                   & 69.0\% & 0.2438 & 0.8986 & 59.9\% & 0.2748 & 0.9110 \\
Gaussian                   & 78.3\% & 0.1784 & 0.5247 & 60.9\% & 0.2326 & 0.6600 \\
Two-anchor-Skywork         & \underline{82.9\%} & \textbf{0.1507} & \textbf{0.4589} & \textbf{61.3\%} & \textbf{0.2284} & \textbf{0.6523} \\
Two-anchor-DeepSeek-Pro    & 80.7\% & 0.1643 & 0.4958 & 60.5\% & 0.2321 & 0.6581 \\
Two-anchor-DeepSeek-Flash  & \textbf{83.1\%} & \underline{0.1582} & \underline{0.4831} & \underline{61.0\%} & \underline{0.2301} & \underline{0.6538} \\

\midrule
\multicolumn{7}{l}{\textbf{\texttt{HelpSteer2-9K}}} \\
\midrule
BT                         & 87.1\% & 0.1297 & 0.4155 & 60.4\% & 0.2480 & 0.7315 \\
BT(Hard)                   & 85.1\% & 0.1410 & 0.4381 & 60.0\% & 0.2351 & 0.6688 \\
Gaussian                   & \underline{88.2\%} & \underline{0.1192} & \underline{0.3897} & 60.4\% & 0.2472 & 0.7540 \\
Two-anchor-Skywork         & 88.0\% & 0.1287 & 0.4082 & \underline{61.5\%} & \underline{0.2303} & \underline{0.6586} \\
Two-anchor-DeepSeek-Pro    & \textbf{88.8\%} & \textbf{0.1169} & \textbf{0.3859} & \textbf{61.8\%} & 0.2335 & 0.6757 \\
Two-anchor-DeepSeek-Flash  & 87.6\% & 0.1317 & 0.4176 & 61.3\% & \textbf{0.2292} & \textbf{0.6508} \\

\bottomrule
\end{tabular}
\end{table*}

\begin{figure}[t]
    \centering
    \includegraphics[width=0.93\linewidth]{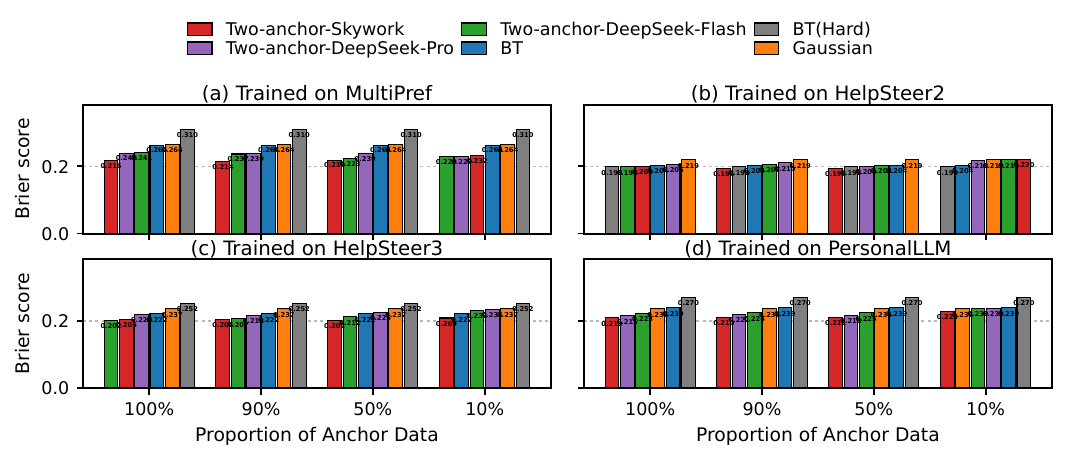}
    \caption{
        \textbf{Anchor data efficiency on RewardBench Brier score.} Results are all based on the Llama-3.2-1B backbone and across all four training datasets.
    }
    \label{fig:anchor_efficiency_Brier_score}
    \vspace{-10pt}
\end{figure}

\begin{figure}[t]
    \centering
    \includegraphics[width=0.93\linewidth]{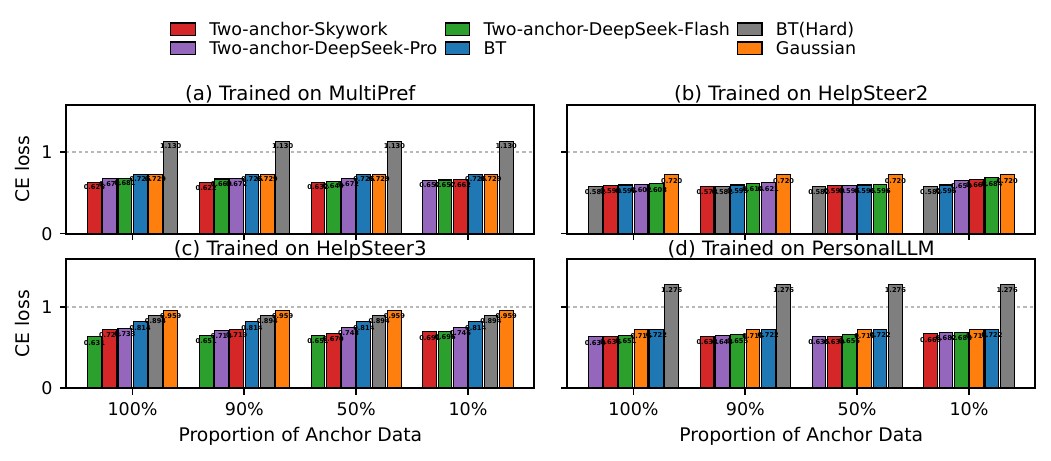}
    \caption{
        \textbf{Anchor data efficiency on RewardBench Cross-entropy.} Results are all based on the Llama-3.2-1B backbone and across all four training datasets.
    }
    \label{fig:anchor_efficiency_CE}
    \vspace{-10pt}
\end{figure}

\end{document}